\documentclass[nohyperref]{article}

\usepackage{microtype}
\usepackage{graphicx}
\usepackage{subfigure}
\usepackage{booktabs}
\usepackage[dvipsnames]{xcolor,colortbl}
\usepackage[accepted]{icml2022}

\usepackage{amsmath,amsfonts,bm}

\def\eqref#1{equation~\ref{#1}}

\def\1{\bm{1}}

\DeclareMathAlphabet{\mathsfit}{\encodingdefault}{\sfdefault}{m}{sl}
\SetMathAlphabet{\mathsfit}{bold}{\encodingdefault}{\sfdefault}{bx}{n}

\DeclareMathOperator*{\argmin}{arg\,min}

\definecolor{mydarkblue}{rgb}{0,0.08,0.45}
\usepackage{hyperref}
\hypersetup{ %
    pdftitle={},
    pdfauthor={},
    pdfsubject={Proceedings of the International Conference on Machine Learning 2022},
    pdfkeywords={},
    pdfborder=0 0 0,
    pdfpagemode=UseNone,
    colorlinks=true,
    linkcolor=mydarkblue,
    citecolor=mydarkblue,
    filecolor=mydarkblue,
    urlcolor=mydarkblue,
    }
\usepackage{url}

\usepackage[utf8]{inputenc}

\usepackage{natbib}
\usepackage{algorithmic}

\usepackage{verbatim}
\usepackage{amsmath,amsthm,amsfonts,amssymb}
\usepackage{amsmath}

\usepackage{mathrsfs}
\usepackage{enumitem}
\usepackage{bm}
\usepackage{multirow}
\usepackage{booktabs}
\usepackage{array, makecell} 
\usepackage{graphicx}
\usepackage{subfigure}
\usepackage{caption}
\usepackage{thmtools}
\usepackage{thm-restate}
\usepackage{dsfont}
\usepackage{xspace}
\usepackage{empheq}
\usepackage{pifont}
\usepackage{tikz}
\usetikzlibrary{positioning,angles,quotes}
\usetikzlibrary{arrows.meta}
\usetikzlibrary{arrows,automata}
\usetikzlibrary{positioning}
\usetikzlibrary{matrix, positioning, arrows.meta}
\tikzset{labels/.style={font=\sffamily\scriptsize},
    circuit/.style={draw,minimum width=2cm,minimum height=\myheight,very thick,inner sep=1mm,outer sep=0pt,cap=round,font=\sffamily\bfseries}
}

\include{notations}
\usepackage[title]{appendix}
\allowdisplaybreaks[4]

\theoremstyle{plain}
\newtheorem{theorem}{Theorem}[section]

\newtheorem{lemma}[theorem]{Lemma}

\theoremstyle{definition}
\newtheorem{definition}[theorem]{Definition}
\newtheorem{assumption}[theorem]{Assumption}
\theoremstyle{remark}

\usepackage[textsize=tiny]{todonotes}

\newenvironment{proofidea}{%
  \proof}{\endproof}

\newcommand{\epsVI}{\epsilon_{\textsc{{\tiny VI}}}\xspace}

\usepackage[utf8]{inputenc}
\usepackage[T1]{fontenc}    
\usepackage{hyperref}       
\usepackage{url}            
\usepackage{booktabs}       
\usepackage{amsfonts}       
\usepackage{nicefrac}       
\usepackage{microtype}      

\usepackage{standalone}

\usepackage{amsmath,amsthm}
\usepackage{mathtools}
\usepackage{amssymb}
\usepackage{bbm}
\usepackage{bm}
\usepackage{enumitem}
\usepackage{dirtytalk}
\usepackage{pifont}
\usepackage{tikz}
\usetikzlibrary{arrows.meta}
\usepackage{float}
\usepackage{cleveref}
\usepackage{hyperref}
\usepackage{soul}
\usepackage{dsfont}
\usepackage{amssymb}
\usepackage[thinc]{esdiff}
\usepackage{fontawesome}

\newcommand{\VISGO}{\texttt{VISGO}\xspace}

\definecolor{vert1}{RGB}{133,146,66} 
\definecolor{vert3}{RGB}{91,140,90} 
\definecolor{vert2}{RGB}{157,193,7} 
\definecolor{vert4}{RGB}{20,200,20} 

\newcommand{\calE}{\mathcal{E}}
\newcommand{\calG}{\mathcal{G}}

\newcommand{\wtcalL}{\wt{\mathcal{L}}\xspace}

\newcommand\myineqi{\mathrel{\stackrel{\makebox[0pt]{\mbox{\normalfont\tiny (i)}}}{\leq}}}
\newcommand\myineqii{\mathrel{\stackrel{\makebox[0pt]{\mbox{\normalfont\tiny (ii)}}}{\leq}}}
\newcommand\myineqiii{\mathrel{\stackrel{\makebox[0pt]{\mbox{\normalfont\tiny (iii)}}}{\leq}}}
\newcommand\myineqiv{\mathrel{\stackrel{\makebox[0pt]{\mbox{\normalfont\tiny (iv)}}}{\leq}}}
\newcommand\myineqv{\mathrel{\stackrel{\makebox[0pt]{\mbox{\normalfont\tiny (v)}}}{\leq}}}

\newcommand\mygei{\mathrel{\stackrel{\makebox[0pt]{\mbox{\normalfont\tiny (i)}}}{\ge}}}

\newcommand\myeqb{\mathrel{\stackrel{\makebox[0pt]{\mbox{\normalfont\tiny (b)}}}{=}}}

\newcommand\mygineeqa{\mathrel{\stackrel{\makebox[0pt]{\mbox{\normalfont\tiny (a)}}}{\geq}}}
\newcommand\mygineeqb{\mathrel{\stackrel{\makebox[0pt]{\mbox{\normalfont\tiny (b)}}}{\geq}}}
\newcommand\mygineeqc{\mathrel{\stackrel{\makebox[0pt]{\mbox{\normalfont\tiny (c)}}}{\geq}}}

\newcommand{\cG}{\mathcal{G}}
\newcommand{\cS}{\mathcal{S}}
\newcommand{\cA}{\mathcal{A}}
\newcommand{\cK}{\mathcal{K}}
\newcommand{\cE}{\mathcal{E}}

\newcommand{\cM}{\mathcal{M}}

\newcommand{\tpi}{{\tilde{\pi}}}
\newcommand{\cKg}{{\mathcal{K},g}}

\makeatletter
\newcommand\footnoteref[1]{\protected@xdef\@thefnmark{\ref{#1}}\@footnotemark}
\makeatother

\newcommand{\para}[1]{\noindent{\bf #1.}}

\def\shownotes{1} 
\ifnum\shownotes=1
\else
\fi

\newcommand{\wt}[1]{\widetilde{#1}}

\newcommand{\wh}[1]{\widehat{#1}}

\DeclarePairedDelimiter\abs{\lvert}{\rvert}%
\DeclarePairedDelimiter\norm{\lVert}{\rVert}%

\let\originalleft\left
\let\originalright\right
\renewcommand{\left}{\mathopen{}\mathclose\bgroup\originalleft}
\renewcommand{\right}{\aftergroup\egroup\originalright}

\usepackage{tcolorbox}
\newtcolorbox{redbox}{colback=red!5!white,colframe=red!75!black}
\newtcolorbox{bluebox}{colback=blue!5!white,colframe=blue!75!black}
\newtcolorbox{yellowbox}{colback=yellow!5!white,colframe=yellow!75!black}

\newcount\Comments  
\definecolor{darkgreen}{rgb}{0,0.5,0}
\definecolor{darkred}{rgb}{0.7,0,0}
\definecolor{teal}{rgb}{0.3,0.8,0.8}

\usetikzlibrary{matrix, positioning, arrows.meta}
\newlength{\myheight}
\tikzset{labels/.style={font=\sffamily\scriptsize},
    circuit/.style={draw,minimum width=2cm,minimum height=\myheight,very thick,inner sep=1mm,outer sep=0pt,cap=round,font=\sffamily\bfseries}
}

\makeatletter
\def\@fnsymbol#1{\ensuremath{\ifcase#1\or *\or \dagger\or \ddagger\or
  \mathsection\or \mathpara\or \|\or \diamond \or **\or \dagger\dagger
  \or \ddagger\ddagger \else\@ctrerr\fi}}
\makeatother

\makeatletter
\newcommand{\printfnsymbol}[1]{%
  \textsuperscript{\@fnsymbol{#1}}%
}
\makeatother

\usepackage{tabularx}
\usepackage{footnote}
\usepackage{hhline}
\usepackage{multirow}

\definecolor{HighlightColor}{gray}{0.97}

\definecolor{colorm}{rgb}{0.2, 0.2, 0.6}
\definecolor{colorn}{rgb}{0.77, 0.12, 0.23}

\newcommand{\PreserveBackslash}[1]{\let\temp=\\#1\let\\=\temp}
\newcolumntype{C}[1]{>{\PreserveBackslash\centering}p{#1}}

\newcommand{\algocomment}[1]{\textcolor{gray}{\textbackslash\textbackslash \textit{#1}}}

\newenvironment{itemize*}%
{\begin{itemize}[leftmargin=*,topsep=0pt]%
		\setlength{\itemsep}{0pt}%
		\setlength{\parskip}{0pt}}%
{\end{itemize}}
\newenvironment{enumerate*}%
{\begin{enumerate}[leftmargin=*,topsep=0pt]%
		\setlength{\itemsep}{0pt}%
		\setlength{\parskip}{0pt}}%
	{\end{enumerate}}

\newcommand{\redmgssp}{\texttt{Re-MG-SSP}}

\newcommand{\mainalgo}{\texttt{VALAE}}
\newcommand{\mainalgofull}{\texttt{Value-Aware Autonomous Exploration}}

\begin{document}

\twocolumn[
\icmltitle{Near-Optimal Algorithms for Autonomous Exploration and Multi-Goal Stochastic Shortest Path}

\begin{icmlauthorlist}
\icmlauthor{Haoyuan Cai}{tsinghua}
\icmlauthor{Tengyu Ma}{stanford}
\icmlauthor{Simon Du}{washington}
\end{icmlauthorlist}
\icmlaffiliation{tsinghua}{Tsinghua University}
\icmlaffiliation{stanford}{Stanford University}
\icmlaffiliation{washington}{University of Washington}

\icmlcorrespondingauthor{Haoyuan Cai}{victorique1929@gmail.com}
\icmlcorrespondingauthor{Tengyu Ma}{tengyuma@stanford.edu}
\icmlcorrespondingauthor{Simon  Du}{ssdu@cs.washington.edu}

\vskip 0.3in
]

\date{}


\newcommand{\cmin}{c_{\textup{min}}}
\newcommand{\RESET}{\textup{RESET}}
\newcommand{\cRESET}{c_{\textup{RESET}}}
\newcommand{\cSL}{\cS_L^{\rightarrow}}

\definecolor{pearThree}{HTML}{E74C3C}
\colorlet{colorast}{pearThree!80!black}

\printAffiliationsAndNotice{} 

\begin{abstract}
    We revisit the incremental autonomous exploration problem proposed by \citet{lim2012autonomous}. In this setting, the agent aims to learn a set of near-optimal goal-conditioned policies to reach the $L$-controllable states: states that are incrementally reachable from an initial state $s_0$ within $L$ steps in expectation.
    We introduce a new algorithm with stronger sample complexity bounds than existing ones. Furthermore, we also prove the first lower bound for the autonomous exploration problem.
    In particular, the lower bound implies that our proposed algorithm, Value-Aware Autonomous Exploration,
    is nearly minimax-optimal when the number of $L$-controllable states grows polynomially with respect to $L$.
Key in our algorithm design is a connection between autonomous exploration and multi-goal stochastic shortest path, a new problem that naturally generalizes the classical stochastic shortest path problem.
    This new problem and its connection to autonomous exploration can be of independent interest.
\end{abstract}

\section{Introduction}

Reinforcement learning (RL) with a known state space has been studied in a wide range of settings  \citep[e.g.,][]{schmidhuber1991possibility,oudeyer2007intrinsic,oudeyer2009intrinsic,baranes2009r}. When the state space is large, it is difficult for a learning agent to discover the whole environment. Instead, the agent can only explore a small portion of the environment. 
At a high level, we hope that the agent can discover states near the initial state, expand the range of known states by exploration, and learn near-optimal goal-conditioned policies for the known states. 
Because the agent discovers its known states of the environment incrementally, this learning problem was named Autonomous Exploration (AX) \citep{lim2012autonomous,tarbouriech2020improved}. 

The autonomous exploration problem generalizes the Stochastic Shortest Path (SSP) problem \citep{bertsekas2000dynamic} where the agent aims to reach a predefined goal state while minimizing its total expected cost. However, in the autonomous exploration setting, the agent aims to discover a set of reachable states in a large environment and find the optimal policies to reach them. 
The autonomous exploration formulation is applicable to an increasing number of real-world RL problems, ranging from navigation in mazes \citep{devo2020deep} to game playing \citep{mnih2013playing}. 
For example, in the maze navigation problem, a robot aims to follow a predefined path in an unknown environment, and the robot has to discover and expand the size of regions known to itself autonomously without prior knowledge of the environment. See \cite{lim2012autonomous} for more discussions.


\para{Related Work} 
The setting of autonomous exploration (AX) was introduced by \citet{lim2012autonomous}, who gave the first algorithm, UcbExplore, with sample complexity $\wt{O}({L^3S^2A}/{\varepsilon^3})$.
Here $L$ denotes the distance within which we hope the learning agent to discover, $S$ denotes the number of states we need to explore,\footnote{In AX, $S$ is often significantly smaller than the size of the entire state space.} $A$ denotes the size of the action space, and $\varepsilon$ denotes the error that we can tolerate. 
Recent work by \citet{tarbouriech2020improved} designed the DisCo algorithm with a sample complexity bound $\wt{O}({L^3S^2A}/{\varepsilon^2})$\footnote{We translate their absolute error $\varepsilon_{\text{abs}}$ to the relative error $\varepsilon_{\text{rel}}$, and $\varepsilon_{\text{abs}} = \varepsilon_{\text{rel}} L$. We will explain the difference of two definitions of $\varepsilon$ in Sect.~\ref{prioralgo}.}, which improves the $1/\varepsilon$ dependency. 
We will briefly discuss the two algorithms in Sect.~\ref{prioralgo}.
In this paper, we present a new algorithm, $\mainalgo$ (Alg.~\ref{algo3}), to further improve the sample complexity, and we also derive the first lower bound.

\setlength{\textfloatsep}{0.2cm}
\colorlet{colorast}{pearThree!80!black}
\begin{table}[t!]
		\centering
		\resizebox{0.8\columnwidth}{!}{%
			\renewcommand{\arraystretch}{2.2}
			\begin{tabular}{|c|c|}
                \hline 
				\textbf{Algorithm} &  \textbf{Sample Complexity}\\
				\hhline{|=|=|}
				\makecell{UcbExplore\\ \citep{lim2012autonomous}}  &  $\wt{O}(L^3S^2A/\varepsilon^3)$
				 \\
            	\hhline{|=|=|} {\makecell{DisCo\\\citep{tarbouriech2020improved}}} & $ \widetilde{O}\left(L^3S^2A/\varepsilon^2\right)$ \\
				\hhline{|=|=|}
			 
			   \rowcolor{HighlightColor} \mainalgo & $  \widetilde{O}\left( LSA / \varepsilon^2 \right)$\\
				\hhline{|=|=|}
				\rowcolor{HighlightColor} Lower Bound  & ${\Omega}( LSA / \varepsilon^2)$\\
				\hline
			\end{tabular}
		}
		\caption{
	\label{tab:comparisons}
		Comparisons between our results and prior results. Algorithms and results in this paper are in grey cells. $L$ is the exploration radius, $A$ is the number of actions, $S$ is the number of states we need to explore, and $\varepsilon$ is the target accuracy. We will define them in Sect.~\ref{sect_preliminaries}.
		For simplicity, we only display the leading term in terms of the  scaling in $1/\varepsilon$. 
}
\end{table}

\subsection{Contributions}
In this paper, we take important steps toward resolving the autonomous exploration problem. 
We compare our results with prior ones in Table~\ref{tab:comparisons}.\footnote{In \citep{lim2012autonomous}, the cost is $1$ uniformly for all state-action pairs. 
In this paper, we allow non-uniform costs. In Table~\ref{tab:comparisons}, we consider uniform costs for fair comparisons.} and we summarize our contributions below:
\begin{enumerate*}
    \item We propose a new algorithm for autonomous exploration problem,  \mainalgofull~(\mainalgo), which uses DisCo algorithm~\cite{tarbouriech2020improved} and \redmgssp~ (cf. Alg.~\ref{algo2}) as initial steps and then uses the estimated value functions to guide our exploration.
    By doing so, for each state-action pair $(s,a)$, we derive an $(s,a)$-dependent sample complexity bound, which can exploit the variance information, and yield a sharper sample complexity bound than the bounds for UcbExplore algorithm and DisCo algorithm (cf. Table~\ref{tab:comparisons}). In particular, \mainalgo~improves the dependency on $L$ from cubic to linear, and improves the dependency on $S$ from square to linear.
    
    \item We connect the autonomous exploration problem to a new problem, multi-goal stochastic shortest path, which generalizes classical SSP. And we show that $\mainalgo$ also applies to multi-goal SSP.
    \item We give the first lower bound of the autonomous exploration problem. 
 This lower bound shows \mainalgo~is nearly minimax-optimal when the number of states we need to explore grows polynomially with respect to  $L$.
\end{enumerate*}

\subsection{Main Difficulties and Technique Overview}
While our work borrows ideas from prior work on autonomous exploration~\citep{lim2012autonomous,tarbouriech2020improved} and recent advances in SSP~\citep{tarbouriech2021stochastic}, we develop new techniques to overcome additional difficulties that are unique in autonomous exploration.



\para{Connection between Autonomous Exploration and Multi-Goal SSP}
In standard RL setting, it is known that in order to obtain a tight dependency on $L$, one needs to exploit the variance information in the value function~\citep{azar2017minimax}.
However, in autonomous exploration, it is unclear how to exploit the variance information because even which state is reachable is unknown.

To this end, we first consider a simpler problem, multi-goal SSP, and extend the technique for single-goal SSP~\citep{tarbouriech2021stochastic} to this new problem (cf. Alg.~\ref{algo3}). We also present a reduction from autonomous exploration to multi-goal SSP (cf. Alg.~\ref{algo2}).
These two techniques together yield the first tight dependency on $L$ for autonomous exploration.

\para{Using Regret to Bound the Sample Complexity} 
To estimate the sample complexity of $\mainalgo$, we need to bound the total number of rounds $r$. Inspired by \citep{lim2012autonomous}, we classify each round into three categories: failure round, success round, and skipped round. Moreover, we adopt the idea of using regret bound.

A failure round has regret larger than $\wt{\Omega}(L/\varepsilon)$, but the number of failure rounds $r_f$ is hard to estimate. The number of success rounds and skipped rounds are bounded by $\wt{O}(SA)$, but the regret in a success round or skipped round can be negative. Hence, to bound the total number of failure rounds $r_f$, careful analyses of both the upper bound and the lower bound of regret are required.

For the upper bound, we extend the techniques of variance analysis from classical SSP (cf. \citep{tarbouriech2021stochastic}) to this problem, and we obtain the upper bound of regret scaling as $\wt{O}(\sqrt{r_f})$. For the lower bound, the total regret in all the failure rounds grows linearly with respect to $r_f$, and we use concentration inequalities to lower bound the total regret in success rounds and skipped rounds (cf. Lem.~\ref{regretlowerbound}.) By solving the inequality that the lower bound of regret is no more than the upper bound, we can obtain an upper bound of $r_f$, and we can finally bound the total number of rounds $r$.

%

\section{Preliminaries}
\label{sect_preliminaries}

\para{Notations} For any two vectors $X,Y \in \mathbb{R}^S$, we write their inner product as $XY := \sum_{s \in \cS} X(s)Y(s)$. We denote $\|X\|_\infty := \max_{s \in \cS} |X(s)|$, and if $X$ is a probability distribution on $\cS$, we define $\mathbb{V}(X, Y) := \sum_{s \in \cS} X(s) Y(s)^2 - (\sum_{s \in \cS} X(s)Y(s))^2$, i.e. the variance of random variable $Y$ over distribution $X$.

\para{Markov Decision Process}
We consider an MDP $M := \langle \mathcal{S}, \mathcal{A}, P, c, s_0 \rangle$, where $\mathcal{S}$ is the state space with size $S$, $\mathcal{A}$ is the action space with size $A$, and $s_0 \in \cS$ is the initial state. In state $s$, taking action $a$ has a cost drawn i.i.d.\,from a distribution on $[\cmin, 1]$ (where $\cmin > 0$) with expectation $c(s, a)$, and transits to the next state $s'$ with probability $P(s' \vert s, a)$.
For convenience, we use $P_{s,a}$ and $P_{s,a,s'}$ to denote $P(\cdot \vert s,a)$ and $P(s' \vert s,a)$,  respectively. A deterministic and stationary policy $\pi : \cS \rightarrow \cA$ is a mapping, and the agent following the policy $\pi$ will take action $\pi(s)$ at state $s$.

For a fixed state $g \in \cS$ we define the random variable $t_g^{\pi}\left(s\right)$ as the number of steps it takes to reach state $g$ starting from state $s$ when executing policy $\pi$, i.e. $t_g^{\pi}\left(s \right):=\inf \left\{t \geq 0: s_{t+1}=g \mid s_{1}=s, \pi\right\}.$ A policy $\pi$ is a proper policy if for any state $s \in \cS$, $t_g^{\pi}\left(s \right) < +\infty$ with probability $1$. Then we define the value function of a proper policy $\pi$ with respect to the goal state $g$ and its corresponding $Q$-function as follows:
\begin{eqnarray}
&V_g^{\pi}(s) = \mathbb{E}\left[\sum_{t=1}^{t_g^{\pi}\left(s \right)} c_{t}\left(s_{t}, \pi\left(s_{t}\right)\right) \mid s_{1}=s \right],\nonumber \\ 
&Q_g^{\pi}(s,a) = \mathbb{E}\left[\sum_{t=1}^{t_g^{\pi}\left(s \right)} c_{t}\left(s_{t}, \pi\left(s_{t}\right)\right) \mid s_{1}=s, \pi(s_1) = a\right], \nonumber
\end{eqnarray}
where $c_t \in [\cmin, 1]$ is the instantaneous cost at step $t$ incurred by the state-action pair $(s_t, \pi(s_t))$, and the expectation is taken over the random sequence of states generated by executing $\pi$ starting from state $s \in \cS$. 
Here we have $V_g^{\pi}(g) = 0$. 
We use $\pi_Q$ to denote the greedy policy over a vector $Q \in \mathbb{R}^{S \times A}$, i.e. $\pi_Q(s) := \argmin\limits_{a\in\cA} Q(s,a)$.

For a fixed state $g \in \cS$, we denote $V^*_g$ as the value function of the optimal policy on MDP $M$ with respect to goal state $g$, and here we list some important properties of $V^*_g$: there exists a stationary, deterministic and proper policy $\pi^*$, such that its value function $V^*_g := V^{\pi^*}_g$ and its corresponding $Q$-function $ Q^*_g := Q^{\pi^*}_g$ satisfies the following Bellman optimality equations (cf. Lem.~\ref{lemma_wellposedproblem}):
\begin{align*}
    Q_g^{{*}}(s,a) = c(s,a) + P_{s,a} V_g^{{*}}, \quad V_g^{{*}}(s) = \min_{a \in \cA} Q_g^{{*}}(s,a).
\end{align*}

We stress that in our setting, given an MDP $M$, the agent knows the state space $\cS$, the action space $\cA$, the constant $\cmin$, but the agent has no prior knowledge of the transition model $P$ or the cost function $c$. In each step $t$, the agent knows its current state $s_t \in \cS$, and taking an action $a_t \in \cA$ will transit to another state $s_t'$ with some cost $c_t$.

\para{{Incrementally $L$-controllable States}}
Before we introduce the Autonomous Exploration problem, we need to define incrementally $L$-controllable states, which are the states we need to explore. 
To formally discuss the setting, we need the following assumption on our MDP $M$.
\newcommand{\assa}{
The action space contains a $\RESET$  action s.t. $P(s_0 \vert s, \RESET) = 1$ for any $s\in \cS$. Moreover, taking $\RESET$ in any state $s$ will incur a cost $\cRESET$ with probability $1$, where $\cRESET$ is a constant in $[\cmin,1]$.
}
\begin{assumption}
\label{assa}
\assa
\end{assumption}

Given any fixed length $L \geq 1$, the agent needs to learn the set of incrementally controllable states $\cS_L^\rightarrow$. To introduce the concept of $\cS_L^\rightarrow$, we first give the definition of policies restricted on a subset:
\begin{definition}[Policy restricted on a subset]
For any $\cS' \subseteq \cS$, a policy $\pi$ is restricted on the set $\cS'$ if $\pi(s) = \RESET$ for all $s \notin \cS'$. 
\end{definition}
Now we discuss the optimal policy restricted on a set of states $\cK \subseteq \cS$ with respect to goal state $g$. We denote $V_\cKg^{{*}} \in \mathbb{R}^S$ as the value function of the optimal policy restricted on $\cK$ with goal $g \in \cS$, and $Q_{\cKg}^{{*}}$ as the $Q$-function corresponding to $V_\cKg^{{*}}$. We consider the case that there exists at least one proper policy restricted on $\cK$ with the goal state $g$. Then, $V_\cKg^{{*}}$ and $Q_{\cKg}^{{*}}$ are finite, and they satisfy the following Bellman equations:
\begin{align*}
Q_{\cKg}^{{*}}(s,a) &= c(s,a) + P_{s,a} V_\cKg^{{*}},  &\forall (s,a) \in \cS \times \cA,\\ 
V_\cKg^{{*}}(s) &= \min\limits_{a\in\cA} Q_{\cKg}^{{*}}(s,a),  &\forall s \in \cK, s\neq g,\\
V_\cKg^{{*}}(s) &=  \cRESET + V_\cKg^{{*}}(s_0),  &\forall s \notin \cK, s\neq g,\\ 
V_\cKg^{{*}}(g) &= 0. & 
\end{align*}

We note that when $\cK_1 \subseteq \cK_2$, for any $g\in\cS$, if $V^*_{\cK_1, g}$ is finite, then $V^*_{\cK_2,g}$ is also finite, and we have $V^*_{\cK_2,g} \leq V^*_{\cK_1, g}$ component-wise. Also, we have $V^*_g = V^*_{\cS,g}$ component-wise.

Now we introduce the definition of incrementally controllable states $\cS_L^\rightarrow$ (see \cite{tarbouriech2020improved} for more intuitions on this definition.):

\begin{definition}[Incrementally $L$-controllable states $\cS_L^\rightarrow$]
Let $\prec$ be any partial order on $\mathcal{S}$. We denote $\mathcal{S}_{L}^{\prec}$ as the set of states reachable from $s_0$ with expected cost no more than $L$ w.r.t. $\prec$, which is defined as follows:
\begin{itemize*}
    \item $s_{0} \in \mathcal{S}_{L}^{\prec}$,
    \item if there is a policy $\pi$ restricted on $\left\{s^{\prime} \in \mathcal{S}_{L}^{\prec}: s^{\prime} \prec s\right\}$ such that $V_s^\pi(s_0) \leq L$, then $s \in \mathcal{S}_{L}^{\prec}$.
\end{itemize*}
The set of incrementally $L$-controllable states  $\mathcal{S}_{L}^\rightarrow$ is given by
$
\mathcal{S}_{L}^{\rightarrow}=\bigcup\limits_{\prec} \mathcal{S}_{L}^{\prec}.
$ And we denote $S_L = |\cSL|$.
\end{definition}


\para{Multi-Goal Stochastic Shortest Path} 
Now we define the multi-goal SSP problem, a natural generalization of the classical SSP problem.
In multi-goal SSP, we consider an MDP $M$ that satisfies  Asmp.~\ref{assa}, and all of its states are incrementally $L$-controllable, i.e. $\cSL = \cS$. Also, we assume that the agent knows $L$.

A learning algorithm for multi-goal SSP takes the error parameter $\varepsilon \in (0,1)$, confidence $\delta \in (0,1)$, and the goal space $\mathcal{G} \subseteq \cS$ as input, and with probability over $1 - \delta$, the algorithm outputs a set of policies $\{\pi_s\}_{s\in\mathcal{G}}$, such that $$\forall s \in \mathcal{G}, V^{\pi_s}_s(s_0) \leq V^*_s(s_0) + \varepsilon L,$$ i.e., the algorithm learns near-optimal policies to reach each $s \in \mathcal{G}$. We note that when the goal space $\calG$ contains a single element, the problem will reduce to classical SSP.

In multi-goal SSP problem, the learning agent interacts with MDP $M$ in this way: the agent knows its current state $s$ and action space $\cA$, but it does not know the model $P(s'\mid s,a)$ and cost function $c(s,a)$. Each time, the agent can choose an action $a \in \cA$, and the agent will observe that it transits to a new state $s'$ with a cost $c$, where $s'$ and $c$ are revealed to the agent. The agent can stop and output the policies anytime when the agent thinks that it has collected enough samples to ensure that it can output near-optimal policies.

The performance of the learning algorithm is measured by the cumulative cost $C_T$, which is defined as follows. We denote $T$ as the total number of steps the agent uses,
and we remark that $T$ is random and chosen by the agent.
We denote $(s_t,a_t)$ as the state-action pair at the $t$-th step. We denote by $c_t(s_t,a_t)$ the instantaneous cost incurred at the $t$-th step. Then we can define 
$C_T := \sum\limits_{t=1}^{T} c_t(s_t,a_t).$

We want to find an algorithm with a probably approximately correct (PAC) bound of $C_T$, i.e., with probability over $1 - \delta$, $C_T$ is bounded by some polynomial of $L,S,A,\varepsilon^{-1},\cmin^{-1}$, and $\log(1/\delta)$.

Here we explain the reason why we need the $\RESET$ action (Asmp.\,\ref{assa}). The classical SSP problem uses an episodic learning protocol, i.e. when the agent reaches the goal state $g$, the agent can "reset" to initial state $s_0$ and start a new episode. But in multi-goal SSP, we do not have episode learning protocol because we need to ensure that for each goal $g \in \calG$,  the agent learns a near-optimal policy to reach $g$. Therefore, each time when the agent arrives at any of the goal, the agent has to “reset” to $s_0$. Hence the $\RESET$ action is necessary, and the previous works \citep{lim2012autonomous} and  \citep{tarbouriech2020improved} also assume the existence of the $\RESET$ action.

We also remark that multi-goal SSP is fundamentally different from reward-free RL \citep{jin2020reward}. Reward-free RL contains two phases: exploration phase and planning phase. In exploration phase we have no knowledge of reward $r$, and in planning phase we cannot interact with MDP. But in multi-goal SSP, we can estimate the cost function $c$, and the agent does not need to separate into two phases.

\para{Autonomous Exploration} 
Now we introduce the autonomous exploration (AX) problem, which generalizes multi-goal SSP. AX problem was first introduced by \citep{lim2012autonomous}, and we use their definition of AX problem.

In AX, we consider an MDP $M$ that satisfies Asmp.\,\ref{assa}. 
A learning algorithm of AX problem inputs the exploration radius $L \geq 1$, the error parameter $\varepsilon \in (0,1)$ and confidence $\delta \in (0,1)$, and with probability over $1 - \delta$, the algorithm should output a set of "known" states $\cK \subseteq \cS$ such that  $\cS_L^\rightarrow \subseteq \cK$, i.e., the algorithm discovers all the states that we want to explore. And the algorithm should also output a set of policies $\{\pi_s\}_{s \in \cK}$, such that 
$$\forall s \in {\cS_L^\rightarrow}, V_s^{\pi_{s}}\left(s_{0}\right) \leq (1+\varepsilon) L,$$
i.e., the algorithm learns a policy to reach each $s \in \cSL$ and the expected cost is no more than $(1+\varepsilon)L$.
In AX, we also use cumulative cost $C_T$ to measure the performance, but we hope $C_T$ depends on $|\cSL|$ instead of the global size $|\cS|$.

We note that different complexity bounds of $C_T$ may depend on $S_L, S_{2L}, S_{(1+\varepsilon)L}$. But if we assume that $S_L$ grows polynomially with respect to $L$, i.e., there exist constants $C,d$ independent of $L$, such that $S_L \leq C L^d$ for all $L \geq 1$, we will have $S_{2L} \leq C 2^d L^d = O(L^d)$, and $S_{(1+\varepsilon)L} = O(L^d)$. Under this assumption, $S_L, S_{2L}, S_{(1+\varepsilon)L}$ are of the same order $O(L^d)$, thus we use $S$ as the abbreviation for all these quantities in Table~\ref{tab:comparisons}. This assumption is implicitly considered in the literature, because otherwise one may need to consider the logarithmic dependency on $S_L$.

In AX, the learning agent does not know the set $\cSL$ or the size of $\cSL$, and it needs to discover and explore $\cSL$ by itself and find policies to reach each state in $\cSL$. This is why the problem is called "autonomous exploration".


We remark that in Sect.\,\ref{sect3}, we will prove that our Alg.\,\ref{algo3} outputs a set $\cK \supseteq \cSL$ and a set of policies $\{\pi_s\}_{s \in \cK}$ restricted on $\cK$, such that
\[\forall s \in {\cK}, V_s^{\pi_{s}}\left(s_{0}\right) \leq V_{\cK,s}^*(s_0) + \varepsilon L.\]

This implies $\forall s \in \cSL$, $V_s^{\pi_{s}}\left(s_{0}\right) \leq (1+\varepsilon) L$, because when $\cSL \subseteq \cK$, we have $V_{\cK,s}^*(s_0) \leq V_{\cSL,s}^*(s_0)$, and for any $s \in \cSL$, we have $V_{\cSL,s}^*(s_0) \leq L$.

In the special case that $\cSL = \cS$ (i.e., in the setting of multi-goal SSP), our Alg.\,\ref{algo3} will output $\cK = \cS$, and the inequality above will be reduced to $\forall s \in {\cS}, V_s^{\pi_{s}}\left(s_{0}\right) \leq V_{s}^*(s_0) + \varepsilon L.$ Hence our Alg.\,\ref{algo3} for AX problem also solves multi-goal SSP problem with goal space $\cG = \cS$.


\subsection{Review of Prior Algorithms}\label{prioralgo}
We review prior algorithms because our algorithm also relies on some components from prior algorithms.

\para{DisCo Algorithm for Autonomous Exploration}

DisCo algorithm was introduced in \citep{tarbouriech2020improved}, and we use DisCo algorithm as a burn-in step for Alg. \ref{algo3}. Here we give the lemma of the sample complexity of DisCo algorithm for autonomous exploration.
\newcommand{\theoremupDisCo}{
Assume that $L \geq 1$, $0 < \varepsilon \leq 1$ and $0 < \delta < 1$. For any MDP $M = \langle \mathcal{S}, \mathcal{A}, P, c, s_0 \rangle$ satisfying Asmp.\,\ref{assa}, with probability at least $1-\delta$, DisCo algorithm will terminate and output a set of states $\cK$ such that $\cS_L^\rightarrow \subseteq \cK \subseteq \cS_{(1+\varepsilon)L}^\rightarrow$, and a set of policies $\{\pi_s\}_{s\in\cK}$ restricted on $\cK$, such that $\forall s\in\cK, V^{\pi_{s}}_s\left(s_{0}\right) \leq V^*_{\cK,s}(s_0) + \varepsilon L,$ and the cumulative cost $C_T=\wt{O}({L^3 S_{(1+\varepsilon)L}^2 A}{\cmin^{-2}\varepsilon^{-2}}).$
}
\begin{lemma}[Corollary 1, \citep{tarbouriech2020improved}]
\label{theoremupDisCo}
\theoremupDisCo
\end{lemma}

Here we clarify that the definitions of $\varepsilon$ in our work and in \citep{tarbouriech2020improved} are different. \citet{tarbouriech2020improved} denotes absolute error as $\varepsilon$ (i.e., they require that the output policies satisfy $V^{\pi_{s}}_s\left(s_{0}\right) \leq L + \varepsilon$ and $V^{\pi_{s}}_s\left(s_{0}\right) \leq V^*_{\cK,s}(s_0) + \varepsilon$), and our paper denotes relative error as $\varepsilon$ (i.e., we require $V^{\pi_{s}}_s\left(s_{0}\right) \leq (1+\varepsilon)L$ and $V^{\pi_{s}}_s\left(s_{0}\right) \leq V^*_{\cK,s}(s_0) + \varepsilon L$). And their absolute error $\varepsilon_{\text{abs}}$ and our relative error $\varepsilon_{\text{rel}}$ satisfies the following equation: $\varepsilon_{\text{abs}} = \varepsilon_{\text{rel}} L$.

We also remark that when $\cmin = 1$, the original form of sample complexity in Theorem 1, \citep{tarbouriech2020improved} was $\wt{O}({L^5 \Gamma_{L+\varepsilon_{\text{abs}}} S_{L+\varepsilon_{\text{abs}}}  A}{\varepsilon_{\text{abs}}^{-2}} + L^3 S_{L+\varepsilon_{\text{abs}}}^2 A{\varepsilon_{\text{abs}}^{-1}})$, where $\Gamma_{L}:=\max\limits_{(s, a) \in \mathcal{S}_{{L}}^\rightarrow \times \mathcal{A}}\left\|\left\{P\left(s^{\prime} \mid s, a\right)\right\}_{s^{\prime} \in \mathcal{S}_L^\rightarrow} \right\|_{0}$, and $\Gamma_{L+\varepsilon_{\text{abs}}} = S_{L+\varepsilon_{\text{abs}}}$ in the worst case. By setting $\varepsilon_{\text{abs}} = \varepsilon_{\text{rel}} L$ and $\Gamma_{L+\varepsilon_{\text{abs}}} = S_{L+\varepsilon_{\text{abs}}}$, we can obtain the sample complexity bound  $\wt{O}({L^3 S_{(1+\varepsilon_{\text{rel}})L}^2 A}{\varepsilon_{\text{rel}}^{-2}})$  in Lem.\,\ref{theoremupDisCo} when $\cmin = 1$. And in Corollary 1, \citep{tarbouriech2020improved}, they discussed the case when $\cmin\in(0,1)$, which incurs an additional $\cmin^{-2}$ in their sample complexity.

\section{Algorithms and Sample Complexity Bounds}
\label{sect3}

Now we are ready to describe our main algorithm $\mainalgo$ (cf. Alg.\,\ref{algo3}), and currently we focus on autonomous exploration problem. There are three key components in Alg.\,\ref{algo3}.
The first component is running DisCo algorithm (cf.  \citep{tarbouriech2020improved}) with $\varepsilon = 1$. Our aim is to discover a set of states $\cK$ such that $\cSL \subseteq \cK \subseteq \cS_{2L}^\rightarrow$, and compute a set of policies $\{\pi_s\}_{s \in \cK}$ to reach each state $s \in \cK$ with expected cost $V_s^{\pi_s}(s_0)$ no more than $2L$. After the first component, we will fix our set $\cK$, and to solve the AX problem, we need only learn a set of policies $\{\pi_s\}_{s\in\cK}$ such that $V_s^{\pi_{s}}\left(s_{0}\right) \leq V_{\cK,s}^*(s_0) + \varepsilon L$ for all $s \in \cK$.

The second component reduces the autonomous exploration problem to multi-goal SSP (cf. Alg.~\ref{algo2}) using the set $\cK$ computed from the first component.
Alg.~\ref{algo2} first constructs a new MDP $M^\dagger$ by "merging" all the states $s \notin \cK$ to a single artificial state $x$, and to solve AX problem, we need only solve multi-goal SSP problem on MDP $M^\dagger$ with goal space $\mathcal{G} = \cK$. Then the algorithm collects fresh samples of the form $(s,a,s',c)$ for all state-action pairs $(s,a) \in \cK \times \cA$, and the aim is to compute the empirical probability $\wh{P}(s'|s,a)$ and the average cost $\wh{c}(s,a)$ with small error.

In the third component, inspired by recent advances in stochastic shortest path~\cite{tarbouriech2021stochastic}, we design a policy evaluation step to obtain near-optimal estimates of the costs of getting to each $s \in \cSL$ (cf. Alg.~\ref{algo3}).

Below we give detailed descriptions for each component.



\setlength{\textfloatsep}{0.2cm}
\begin{algorithm}[tb]
\caption{\textbf{Re}duce Autonomous Exploration to \textbf{M}ulti-Goal \textbf{SSP} (\redmgssp)}
\label{algo2}
\small
\begin{algorithmic}[1]

\STATE {\bf Input:}
Confidence $\delta \in(0,1)$, exploration radius $L \geq 1$,

\STATE {\bf Input:} a set of states $\cK$, and a set of policies $\{\pi_{s}\}_{s\in\cK}$.\\

\STATE Define MDP $M^\dagger = \langle \cK^\dagger, \mathcal{A}, P^\dagger, c^\dagger, s_0 \rangle$ where $\cK^\dagger$, $P^\dagger$, $c^\dagger$ are defined in Sect.~\ref{sect_algo2}.

\STATE $\forall (s,a,s') \in \cK^\dagger \times \cA \times \cK^\dagger$, set $N(s,a,s')\leftarrow0;~ \wh P_{s,a,s'}\leftarrow0.$

\STATE $\forall (s,a) \in \cK^\dagger \times \cA$, set $N(s,a) \leftarrow 0;~ n(s,a) \leftarrow 0;~ \theta(s,a) \leftarrow 0;~ \wh c (s,a) \leftarrow 0$. \\

\STATE Set $\psi \leftarrow 12000{L^2|\cK|}{\cmin^{-2}}\ln(\frac{|\cK|A}{\delta})$, and $\phi \leftarrow 2^{\lceil \log_2 \psi\rceil}$.

\FOR{\textup{each} $(s,a) \in \cK \times \mathcal{A}$}{

\WHILE{$N(s,a) < \phi$}
{
\STATE Execute policy $\pi_{s}$ on MDP $M^\dagger$ until reaching state $s$.\\

\STATE Take action $a$, incur cost $c$ and observe next state $s' \sim P^\dagger\left(\cdot \mid s, a\right)$.\\

\STATE Set $N(s,a) \leftarrow N(s,a) + 1$, $\theta(s,a) \leftarrow \theta(s,a) + c$,  $N(s,a,s') \leftarrow N(s,a,s') + 1$.\\
}
\ENDWHILE

\STATE Set $\wh{c}(s,a) \leftarrow  \frac{ \theta(s,a)}{N(s,a)}$ and $\theta(s,a) \leftarrow 0$. \\
For all $s^{\prime} \in \cK^\dagger$, set $n(s, a) \leftarrow N(s, a),~ \widehat{P}_{s, a, s^{\prime}} \leftarrow N\left(s, a, s^{\prime}\right) / N(s, a)$.\\
}
\ENDFOR

\STATE For all $a \in \cA$, set $N(x,a) \leftarrow \phi, n(x,a) \leftarrow \phi$, $\wh{c}(x,a) \leftarrow \cRESET$, $\wh{P}_{x,a,s_0} \leftarrow 1$.

\STATE For all $a \in \cA$, $s' \in \cS$,  set $\wh{P}_{x,a,s'} \leftarrow 0$.




\STATE {\bf Output:} $N(), n(), \wh{P}, \theta(), \wh{c}$.
\end{algorithmic}
\end{algorithm}

\subsection{Running DisCo Algorithm with $\varepsilon = 1$}

In the first component of our main algorithm $\mainalgo$ (cf. Alg.\,\ref{algo3}), we use DisCo algorithm with (relative) error $\varepsilon = 1$ as a subroutine. By Lem.\,\ref{theoremupDisCo}, we can obtain a set $\cK$ such that $\cSL \subseteq \cK \subseteq \cS_{2L}^\rightarrow$, and a set of policies $\{\pi_s\}_{s \in \cK}$ such that $\forall s\in\cK, V_s^{\pi_s}(s_0) \leq 2L,$ and the total cost is bounded by $\wt{O}({L^3 S_{(1+\varepsilon)L}^2 A}{\cmin^{-2}})$. In the next subsection, we will focus on a fixed set $\cK$, and reduce the autonomous exploration problem to multi-goal SSP problem.

\subsection{Connection between Autonomous Exploration and Multi-Goal SSP} 
\label{sect_algo2}
In our main algorithm $\mainalgo$ (Alg.\,\ref{algo3}),  after running DisCo with $\varepsilon = 1$, we have obtained a set of known states $\cK \supseteq \cSL$ and discovered all the states that we want to explore, and we denote $K = |\cK|$. Now we focus on the second component of $\mainalgo$ (cf. Alg.\,\ref{algo2}).
We will fix our set of known states $\cK$, and focus only on the policies restricted on $\cK$. Therefore, for all the states $s \notin \cK$, we can regard them as one artificial state $x$, and the only action at state $x$ is \RESET. To this purpose, we will construct an MDP $M^\dagger := \langle \cK^\dagger, \mathcal{A}, P^\dagger, c^\dagger, s_0 \rangle$ where we first define the artificial state $x$, and we set $\cK^\dagger = \cK \cup \{x\}$, and we denote $K' = |\cK^\dagger| = K+1$. For any $(s,a) \in \cK \times \cA$, we define $P^\dagger_{s,a,s'}$ as follows:
\vspace{-0.2cm}
\begin{align*}
P^\dagger_{s,a,s'} = P_{s,a,s'},  ~\forall s' \in \cK, \text{ and }
P_{s,a,x}^\dagger =  \sum\limits_{s' \notin \cK} P_{s,a,s'}.
\vspace{-0.2cm}
\end{align*}
We also define $P_{x,a,s'}^\dagger = \mathds{I}[s' = s_0]$ for any $a \in \cA, s' \in \cK^\dagger$. Finally, we define $c^\dagger(s,a) = c(s,a)$ for all $(s,a) \in \cS \times \cA$, and $c^\dagger(x,a) = \cRESET$ for all $a \in \cA$.
In this way, the AX problem reduces to multi-goal SSP problem on MDP $M^\dagger$ with the set of states being $\cK^\dagger$ and goal space $\cG = \cK$, and all states in $\cK$ are incrementally $2L$-controllable from $s_0$.

Next, we collect $\phi = \wt{\Omega}(L^2|\cK|/\cmin^2)$ fresh samples for each state-action pair $(s,a)\in \cK \times \cA$. Our aim is that for each state-action pair $(s,a) \in \cK \times \cA$, we can obtain $\phi$ samples of the form $(s,a,s',c)$ and compute the empirical probability $\wh{P}(s'|s,a)$ and the average cost $\wh{c}(s,a)$, so that our estimation $\wh{P}(s'|s,a)$ and $\wh{c}(s,a)$ are close enough to ${P}^\dagger(s'|s,a)$ and $c^\dagger(s,a)$, respectively. In DisCo algorithm, we have computed a policy $\pi_s$ for each $s \in \cK$, such that we can execute $\pi_s$ to reach state $s$ from $s_0$ with expected cost no more than $2L$. Hence, to obtain a sample $(s,a,s',c)$ at any state-action pair $(s,a) \in \cK \times \cA$, we need only first execute $\pi_s$ to arrive at state $s$, then execute action $a$.

We remark that using fresh samples is essential for Alg.~\ref{algo3} to ensure these samples are independent of $\cK$, and we cannot use the samples collected in DisCo algorithm because they are dependent of $\cK$. Also, we note that in Alg.\,\ref{algo2} and Alg.\,\ref{algo3}, the estimated transition probability $\wh{P}(s'\mid s,a)$ and the estimated cost $\wh{c}(s' \mid s,a)$ are only evaluated for all $(s,a) \in \cK^\dagger \times \cA$ on MDP $M^\dagger$, rather than for all $(s,a) \in \cS \times \cA$ on MDP $M$, hence the computational complexity of Alg.\,\ref{algo2} and Alg.\,\ref{algo3} does not depend on "global" $|\cS|$. 

We note that the idea of uniformly connecting $\phi$ samples for each state-action pair $(s,a) \in \cK \times \cA$ is similar with DisCo algorithm. The difference is that DisCo algorithm collects $\wt{\Omega}(L^2 |\cK| \cmin^{-2} \varepsilon^{-2})$ samples for each state-action pair $(s,a) \in \cK \times \cA$, but in Alg.\,\ref{algo2} our $\phi = \wt{\Omega}(L^2 |\cK| \cmin^{-2})$ and is smaller than that in DisCo.

\setlength{\textfloatsep}{0.2cm}
\begin{algorithm}[!t]
\caption{\textbf{Val}ue-Aware \textbf{A}utonomous \textbf{E}xploration (\mainalgo)}
\label{algo3}
\begin{algorithmic}[1]

        \small
\STATE {\bf Input:}
Confidence  $\delta \in(0,1)$, error $\varepsilon \in (0,1],$ and $L \geq 1$.

\STATE {\bf Input (for multi-goal SSP only):}
Goal Space $\cG \subseteq \cS$. 

\STATE (For autonomous exploration, set $\cG = \emptyset$.)

\STATE \textbf{Specify:} Trigger set $\mathcal{N} \leftarrow \{ 2^{j-1}\ :\ j=1,2,\ldots\}$. \\ 


\algocomment{We run DisCo algorithm with $\varepsilon = 1$ and get a set $\cK$ such that $\cSL\subseteq\cK\subseteq\cS_{2L}^{\rightarrow}$.}

\STATE Run DisCo algorithm with input $(\delta,\varepsilon=1,L)$ and we get a set $\cK$ and a set of policies $\{\pi_s\}_{s\in\cK}$.\\

\STATE Run Alg.\,\ref{algo2} with input $(\delta,L,\cK,\{\pi_s\}_{s\in\cK})$, and we obtain the variables $N(), n(), \wh{P}, \theta(), \wh{c}$.
\\


\STATE	Set time step $t \leftarrow 1$ and trigger index $j \leftarrow 5+\log_2\frac{1}{\cmin}$.
	
\STATE Set $\epsilon \leftarrow {\varepsilon}/{3}$, $B \leftarrow  10L$, $
\lambda = \wt{O}(1/\epsilon^2),$ and $g \leftarrow s_{0}$.
 \\

\STATE Initialize $\mathcal{G} \leftarrow \cK$ if $\mathcal{G} = \emptyset$. \\

\STATE \algocomment{Solve multi-goal SSP problem on $M^\dagger$ with goal space $\mathcal{G}$.}

\FOR{\textup{round} $r = 1,2,\cdots$}
\STATE \algocomment{Phase (a): Compute Optimal Policy}

        \STATE Compute $ (Q, V) := \mathtt{VISGO}(g, 2^{-j} / (|\cK^\dagger| A))$.\\
        
        \STATE Set the policy $\tpi$ as the greedy policy over $Q$, and $\hat{\tau} \leftarrow 0$.

\STATE \algocomment{Phase (b): Policy Evaluation}
\FOR{\textup{episode} $k = 1,2,\cdots,\lambda$}

\STATE Set $s_{t} \leftarrow s_{0}$ and reset to the initial state $s_0$, and $\hat{\tau}_k \rightarrow 0.$\\

\WHILE{$s_t \neq g$}

{
\STATE Take action $a_{t}=\arg\min_{a \in \mathcal{A}} Q\left(s_{t}, a\right)$ on $M^\dagger$, incur cost $c_t$ and observe next state $s_{t+1} \sim P^\dagger\left(\cdot \mid s_{t}, a_{t}\right)$.\\

\STATE $\operatorname{Set} \  \left(s, a, s^{\prime}, c\right) \leftarrow\left(s_{t}, a_{t}, s_{t+1}, c_t \right)$ and $t \leftarrow t+1$.\\

\STATE $\operatorname{Set}   N(s, a) \leftarrow N(s, a)+1$, $ \theta(s,a) \leftarrow \theta(s,a) + c$, $N\left(s, a, s^{\prime}\right) \leftarrow N\left(s, a, s^{\prime}\right)+1$.\\


\IF{$N(s, a) \in \mathcal{N}$}
{
\STATE Set $j \leftarrow j+1$, $\wh{c}(s,a) \leftarrow  \frac{2 \theta(s,a)}{N(s,a)}$ and $\theta(s,a) \leftarrow 0$.\\
\STATE For all $s^{\prime} \in \cK^\dagger$, set $n(s, a) \leftarrow N(s, a), \widehat{P}_{s, a, s^{\prime}} \leftarrow N\left(s, a, s^{\prime}\right) / N(s, a)$.\\
\STATE Return to line 11, start a new round (the current round has been a \emph{skipped round}).
}
\ENDIF

\STATE Set $\hat{\tau} \leftarrow \hat{\tau} + \frac{c}{\lambda}, \hat{\tau}_k \leftarrow \hat{\tau}_k + c$.

}\ENDWHILE

\IF{$\hat{\tau} > V(s_0) + \epsilon L$}
{
\STATE Return to line 11, start a new round. (the current round has been a \emph{failure round}).
}
\ENDIF
\ENDFOR

\STATE Set $\pi_g \leftarrow \tpi$. Remove $g$ from $\mathcal{G}$. (The current round has been a \emph{success round}.)

\STATE Choose another state $g \in \mathcal{G}$. 

\STATE Stop the algorithm if $\mathcal{G}$ is empty.
\ENDFOR

{\bf Output:}
	The states $s$ in $\cK$ and their corresponding policy $\pi_{s}$.

\end{algorithmic}
\end{algorithm}

\subsection{Value-Aware Algorithms for Autonomous Exploration and Multi-Goal SSP} 
Finally we describe our main algorithm, \mainalgofull~(\mainalgo, cf. Alg.~\ref{algo3}).
First, \mainalgo~uses DisCo algorithm with $\varepsilon = 1$ as a subroutine, and DisCo algorithm computes a set $\cK$ such that $\cSL \subseteq \cK$. We discard all the samples collected in DisCo algorithm, in order to ensure the independence of $\cK$ and $\wh{P}_{s,a}$. 
Second, we use Alg.\,\ref{algo2} as a burn-in step to collect $\wt{\Omega}(L^2|\cK|/\cmin^2)$ samples for each of the state-action pair $(s,a)$ so that the empirical model $\wh P$ and the true model $P^\dagger$ are close enough.
This guarantees that with high probability, in any round $r$, the expected cost of the greedy policy $\tilde{\pi}$ in Phase (a) on model $P^\dagger$ is no more than ${O}(L)$, which is proved in Lem.~\ref{consterror}.

From now on we work on the MDP $M^\dagger$, and we will solve the multi-goal SSP problem on $M^\dagger$ and compute near-optimal policies $\pi_g$ for all the goal states $g \in \cK$. We choose the goal state $g \in \cK$ one by one, and we move to another goal state $g$ if the average performance of the policy $\pi_g$ is close to our estimation of the optimal policy.
In each round, we have two phases. In the first phase, we use \VISGO (cf. Alg.~\ref{plan} in Appendix~\ref{appendix_VISGO}) to estimate the value function of the optimal policy with goal state $g$ (denoted as $V$), and we set the policy $\tpi$ as the greedy policy over its output $Q$. We note that $V$ is optimistic, i.e., $V(s) \leq V^*_{\cK,g}(s) \leq 2L+1$. Since we do not know whether the policy $\tpi$ is close enough to the optimal policy, in the second phase, we will execute $\tpi$ for $\lambda = \wt{O}({1}/{\epsilon^2})$ times and check whether the average performance is close enough to our estimation of the optimal cost (i.e., check whether $\hat{\tau} \leq V(s_0) + \epsilon L$). By setting
$$
\lambda = \lceil \frac{2048}{\epsilon^2}\ln^2(\frac{256}{\epsilon})\ln(\frac{2|\cK|}{\delta}) \rceil
$$ and using concentration inequalities (Lem.\,\ref{goodpolicy}), we can prove that the average performance $\hat{\tau}$ in $\lambda$ episodes is close enough to the expected cost of $\tilde{\pi}$. 
In this process, we also collect samples, and use them to help us estimate the value function of the optimal policy.

In the second phase, the current round will be classified into three cases: failure round, skipped round, and success round. This borrows the idea from \citep{lim2012autonomous}. If the average performance of the policy $\tpi$ is too bad (i.e., $\hat{\tau}$ is larger than $V(s_0) + \epsilon L$), we will consider the current round as a failure round. If the number of samples $N(s,a)$ meets the trigger set (i.e. is a power of 2), we will consider the current round as a skipped round, following the idea in \citep{jaksch2010near}. Otherwise, the current round is a success round. In the case of a failure round or a skipped round, we will not change the goal state $g$, and in the next round, we compute a new policy by \VISGO using the samples collected in this round. In the case of a success round, as the average performance of the policy $\tpi$ is close to optimal, we can set the $\tpi$ as the policy $\pi_g$ for the goal state $g$, and choose another goal state $g$.



\newcommand{\theoremupSSP}{

Assume that $L \geq 1$, $0 < \varepsilon \leq 1$ and $0 < \delta < 1$. For any MDP $M = \langle \mathcal{S}, \mathcal{A}, P, c, s_0 \rangle$ satisfying Asmp.\,\ref{assa}, with probability at least $1-\delta$, our Alg.\,\ref{algo3} will terminate and output a set of states $\cK$ such that $\cS_L^\rightarrow \subseteq \cK \subseteq \cS_{2L}^\rightarrow$, and a set of policies $\{\pi_s\}_{s\in\cK}$ restricted on $\cK$, such that $\forall s \in \cK, V^{\pi_{s}}_s\left(s_{0}\right) \leq V^*_{\cK,s}(s_0) + \varepsilon L,$ and the cumulative cost $C_T = \wt{O}({L  S_{2L} A}{\varepsilon^{-2}}+{L S_{2L}^{2} A}{\varepsilon^{-1}} + {L^3 S_{2L}^2 A}{\cmin^{-2}}).$ 
And when $\varepsilon \leq \min(S_{2L}^{-1}, L^{-1}\cmin)$, we have $C_T = \wt{O}({L  S_{2L} A}{\varepsilon^{-2}})$.
}
\begin{theorem}[Cumulative Cost for AX]
\label{theoremupSSP}
\theoremupSSP
\end{theorem}

Thm.\ref{theoremupSSP} shows that Alg.\ref{algo3} solves autonomous exploration problem.
Note that in Thm.\,\ref{theoremupSSP}, the dependency on $L$ is tight when $\varepsilon \rightarrow 0$, because we leverage the variance information in the policy-evaluation phase, which is necessary in RL problems generally. 
DisCo algorithm does not use the variance information because it collects equal number of samples on each state-action pair $(s,a)$, i.e., the sample collection in DisCo algorithm does not use the estimated value function as the guidance.



We highlight that the leading term of $C_T$ does not have $\cmin$. This is because the variance fundamentally does not scale with $\cmin$
(cf. Lem.\,\ref{regretbound} and Lem.\,\ref{regretlowerbound}). While we discover a larger set $\cK \subseteq \cS_{2L}^\rightarrow$ compared with \citep{lim2012autonomous} and \citep{tarbouriech2020improved}, we note that if the number of the $L$-controllable states grows polynomially with respect to $L$,  $S_L$ and $S_{2L}$ will be of the same order. Hence under this assumption, our sample complexity bound strictly improves the existing ones and is nearly minimax optimal.

Lastly, we note that Alg.~\ref{algo3} also solves the multi-goal SSP problem, and it enjoys a near-optimal sample complexity for multi-goal SSP:
\begin{theorem}[Cumulative Cost for Multi-Goal SSP]
Assume that $L \geq 1$, $0 < \varepsilon \leq 1$, $0 < \delta < 1$ and goal space $\mathcal{G} \subseteq \cS$. For any MDP $M = \langle \mathcal{S}, \mathcal{A}, P, c, s_0 \rangle$ satisfying Asmp.\,\ref{assa} and $\cSL = \cS$,
with probability at least $1-\delta$, our Alg.\,\ref{algo3} will terminate and output a set of policies $\{\pi_s\}_{s\in\mathcal{G}}$ such that $\forall s \in \mathcal{G}, V^{\pi_{s}}_s\left(s_{0}\right) \leq V^*_{s}(s_0) + \varepsilon L,$ and the cumulative cost 
$C_T = \wt{O}({L  S A}{\varepsilon^{-2}}+{L S^{2} A}{\varepsilon^{-1}} + {L^3 S^2 A}{\cmin^{-2}}).$ And when $\varepsilon \leq \min(S^{-1}, L^{-1}\cmin)$, we have $C_T = \wt{O}({L S A}{\varepsilon^{-2}})$.
\end{theorem}

\section{A Minimax Lower Bound for Autonomous Exploration}

\begin{tikzpicture}[->,>=stealth',shorten >=1pt,auto,node distance=2.8cm, semithick]
\tikzstyle{every state}=[fill=white,text=black]

\node[state]                (W)                  {$s_0$};
\node[state]                (A) [right=1.5cm of W]    {$s_1$};
\node[state, fill=white]    (B) [right=2.9cm of A] {$g$};


\path  
	(W)  edge node{$\frac{2}{L}$} (A)
	(W)  edge [loop below] node{$1-\frac{2}{L}$} (W)  
	(A)  edge [loop above, dashed] node{$1 - \frac{2}{(1+6\varepsilon)L}$} (A)
	(A)  edge [loop below, blue] node{$1 - \frac{2}{L}$} (A)
	(A)  edge[blue, bend right]   node[below]{$\textcolor{blue}{\frac{2}{L}}$} (B)
	(A)  edge[dashed]   node[above]{$\frac{2}{(1+6\varepsilon)L}$} (B)
	(B)  edge [loop right] node{1} (B)  
	;

\end{tikzpicture}


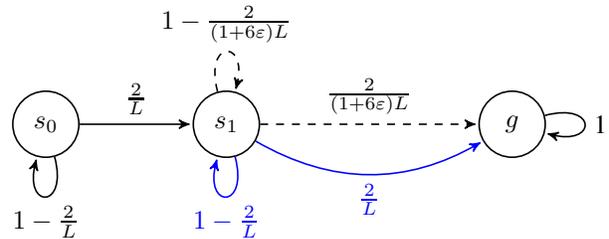
\captionof{figure}{Illustration of our construction of the hard MDP.}
\label{FigureLB}

\newcommand{\algo}{\boldsymbol{\pi}}

Here we present our lower bound of sample complexity for the autonomous exploration problem, and we follow the definitions in \cite{domingues2021episodic}.

We define a learning algorithm as a history-dependent policy $\algo$ used to interact with an MDP $M$, and the rigorous definition of $\algo$ is in Appendix~\ref{app_ax_lowerbound}. We recall that in AX setting, the algorithm eventually stops and output a set $\cK \subseteq \cS$ and a set of policies $\{\pi_s\}_{s \in \cK}$. 
Hence we define an algorithm for the AX problem as a tuple $(\algo, \tau, \cK, \{\pi_s\}_{s \in \cK})$, where $\tau$ is the stopping time (total number of steps) chosen by the algorithm, $\cK$ and $\{\pi_s\}_{s \in \cK}$ are the output of the algorithm.  Now we formally write the definition of an algorithm for AX problem.

\begin{definition}
    An algorithm $(\algo, \tau, \cK, \{\pi_s\}_{s \in \cK})$ is $(\varepsilon, \delta, L)$-PAC for AX problem on MDP $M$, if with probability over $1 - \delta$, the algorithm returns a set of states $\cK$ and a set of policies $\{\pi_s\}_{s \in \cK}$ after $\tau$ steps, such that $\cK \supseteq \cSL$ and $\forall s \in \mathcal{S}_L^\rightarrow, V^{\pi_{s}}_s\left(s_{0}\right) \leq (1+\varepsilon)L$.
\end{definition}

We note that $\tau$ is a random variable over the probability distribution $\mathbb{P}_{\boldsymbol{\pi}, M}$, where $\mathbb{P}_{\boldsymbol{\pi}, M}$ is determined by algorithm $\algo$ and MDP $M$, and $\mathbb{P}_{\boldsymbol{\pi}, M}$ is defined in Appendix~\ref{app_ax_lowerbound}. Also, we denote the operator $\mathbb{E}_{\boldsymbol{\pi}, M}$ as the expectation under $\mathbb{P}_{\boldsymbol{\pi}, M}$.

Then for any real numbers $L, \cmin$ and positive integers $S,A,S_L$, we define a class of MDPs $\mathfrak{M}(L,S_L)$ as follows: $\mathfrak{M}(L,S_L)$ contains all the MDPs $M =  \langle \mathcal{S}, \mathcal{A}, P, c, s_0 \rangle$, such that $|\cS| \leq S$, $|\cA| \leq A$, $c(s,a) \in [\cmin,1]$ for all $(s,a) \in \cS \times \cA$, and $M$ satisfies Asmp.~\,\ref{assa} and $|\cSL| \leq S_L$.

Finally, the following theorem states the lower bound for the autonomous exploration problem.
\newcommand{\lowerbound}{
Assume that $L > 4$, $S > 8$, $A > 4$, $4 \leq S_L \leq \min\{(A - 1)^{\lfloor\frac{L}{2}\rfloor},S\}$, $0 < \varepsilon < \frac{1}{4}$, $0 < \delta < \frac{1}{16}$, and $0 < \cmin \leq 1$. 
Then for any algorithm $(\algo, \tau, \cK, \{\pi_s\}_{s \in \cK})$ that is $(\varepsilon, \delta, L)$-PAC for AX problem on any MDP $M \in \mathfrak{M}(L,S_L)$, there exists an MDP $\cM \in \mathfrak{M}(L,S_L)$ such that 
\[
\mathbb{E}_{\algo, \cM}[\tau] =  \Omega(\frac{L S_L A}{\cmin \varepsilon^2} \log \frac{1}{\delta}).
\]
}
\begin{theorem}
\label{lowerbound}
\lowerbound

\end{theorem}

As $\tau$ is the total number of steps used in the algorithm $\algo$, the lower bound of cumulative cost $C_T$ is $\cmin$ multiplies the lower bound of $\tau$, i.e., $\Omega({L S_L A}{ \varepsilon^{-2}} \log \frac{1}{\delta})$.
This lower bound further implies our upper bound (Theorem~\ref{theoremupSSP}) is nearly minimax-optimal when $S_L$ and $S_{2L}$ are of the same order.
We also have a lower bound for multi-goal SSP (cf. Appendix \ref{sect_lowerbound_MGSSP}).

\subsection{Proof Sketch}
We briefly sketch our proof of the lower bound. We consider the case $\cmin = 1$ and $L > 2$ for convenience, and we first construct our family of hard MDPs for $S = 3$ states (cf. Fig.~\ref{FigureLB}), where $s_0$ is initial state, $s_1$ is middle state and $g$ is goal state. In the initial state $s_0$, taking any action $a$ will transit to state $s_1$ with probability $\frac{2}{L}$, and stay at state $s_0$ with probability $1 - \frac{2}{L}$. In state $s_1$, there is only one optimal action $a^*$. When we take the action $a^*$ in $s_1$ (the blue edges), the agent will transit to the goal state $g$ with probability $\frac{2}{L}$ and stay at $s_1$ with probability $1 - \frac{2}{L}$. When we take an action $a \neq a^*$ in $s_1$ (the dashed edges), the agent will transit to the goal state $g$ with smaller probability $\frac{2}{(1+6\varepsilon)L}$. We note that the $\RESET$ action is not drawn in Fig.~\ref{FigureLB}.

We can verify that $V^*_{g} (s_0) = \frac{L}{2} + \frac{L}{2} = L$, and $g \in \cSL$ (hence $g$ should be contained in the learning algorithm's output $\cK$). Let $\pi_{g}$ be the output policy of the learning algorithm with respect to goal state ${g}$. If $\pi_{g}(s_1) = a^*$, we have $V^{\pi_{g}}_{g}(s_0) = L$. Otherwise, we have $V^{\pi_{g}}_{g}(s_0) = \frac{L}{2} + \frac{(1+6\varepsilon)L}{2} > (1 + \varepsilon) L$, i.e., the policy $\pi_g$ is not valid output for AX if $\pi_g(s_1) \neq a^*$. Hence, if the algorithm solves AX problem on this MDP, it has to discriminate between two Bernoulli distributions with $p_1 = \frac{2}{L}$ and $p_2 = \frac{2}{(1+6\varepsilon)L}$ among all the $A$ actions, and the KL divergence of the two distributions is $O(\varepsilon^2/L)$. Hence we can prove that we need at least $\wt{\Omega}(LA/\varepsilon^2)$ to solve AX on this MDP. The technique of KL divergence is similar with \cite{domingues2021episodic}.

Then we can generalize our hard MDP to larger $S_L$. We first construct an MDP $\cM_0'$ with $ {S_L}-1 $ states, and each middle state $s_i$ can be reached from $s_0$ in $L/2$ steps in expectation. 
Then we add a goal states $g$, and we choose one optimal state-action pair $(s_i^*, a^*)$ among all the middle states and actions. Finally, we set the transition probability $P(g|s_i^*, a^*) = \frac{2}{L}$, and $P(g|s_i, a) = \frac{2}{(1+6\varepsilon)L}$ for other pair of middle state and action $(s_i, a)$. In intuition, this extends the construction in Fig.~\ref{FigureLB} from $A$ actions to $O(S_L A)$ actions. The full construction is in Appendix~\ref{app_ax_lowerbound}. Under this construction, we can prove that the lower bound scales as $\wt{\Omega}(LS_L A/\varepsilon^2)$.

\section{Conclusion}
We introduced a new algorithm for the autonomous exploration problem, which improves existing ones.
Along the way, we also introduced a new problem, multi-goal SSP problem, which can be of independent interest.
The natural future directions include designing an algorithm with $\widetilde{O}\left(\frac{LS_{L}A}{\varepsilon^2}\right)$ sample complexity instead of $\widetilde{O}\left(\frac{LS_{2L}A}{\varepsilon^2}\right)$,  and improving the lower order terms in existing bounds. 


\section*{Acknowledgements}
The work is supported by JD.com. Simon S. Du gratefully acknowledges the funding from NSF Award’s IIS-2110170 and DMS-2134106.

\bibliographystyle{icml2022}
\bibliography{bibliography}

\begin{thebibliography}{18}
\providecommand{\natexlab}[1]{#1}
\providecommand{\url}[1]{\texttt{#1}}
\expandafter\ifx\csname urlstyle\endcsname\relax
  \providecommand{\doi}[1]{doi: #1}\else
  \providecommand{\doi}{doi: \begingroup \urlstyle{rm}\Url}\fi

\bibitem[Azar et~al.(2017)Azar, Osband, and Munos]{azar2017minimax}
Azar, M.~G., Osband, I., and Munos, R.
\newblock Minimax regret bounds for reinforcement learning.
\newblock In \emph{Proceedings of the 34th International Conference on Machine
  Learning-Volume 70}, pp.\  263--272. JMLR. org, 2017.

\bibitem[Baranes \& Oudeyer(2009)Baranes and Oudeyer]{baranes2009r}
Baranes, A. and Oudeyer, P.-Y.
\newblock R-iac: Robust intrinsically motivated exploration and active
  learning.
\newblock \emph{IEEE Transactions on Autonomous Mental Development}, 1\penalty0
  (3):\penalty0 155--169, 2009.

\bibitem[Bertsekas \& Tsitsiklis(1991)Bertsekas and
  Tsitsiklis]{bertsekas1991analysis}
Bertsekas, D.~P. and Tsitsiklis, J.~N.
\newblock An analysis of stochastic shortest path problems.
\newblock \emph{Mathematics of Operations Research}, 16\penalty0 (3):\penalty0
  580--595, 1991.

\bibitem[Bertsekas et~al.(2000)]{bertsekas2000dynamic}
Bertsekas, D.~P. et~al.
\newblock \emph{Dynamic programming and optimal control: Vol. 1}.
\newblock Athena scientific Belmont, 2000.

\bibitem[Cohen et~al.(2020)Cohen, Kaplan, Mansour, and
  Rosenberg]{cohen2020near}
Cohen, A., Kaplan, H., Mansour, Y., and Rosenberg, A.
\newblock Near-optimal regret bounds for stochastic shortest path.
\newblock \emph{arXiv preprint arXiv:2002.09869}, 2020.

\bibitem[Devo et~al.(2020)Devo, Costante, and Valigi]{devo2020deep}
Devo, A., Costante, G., and Valigi, P.
\newblock Deep reinforcement learning for instruction following visual
  navigation in 3d maze-like environments.
\newblock \emph{IEEE Robotics and Automation Letters}, 5\penalty0 (2):\penalty0
  1175--1182, 2020.

\bibitem[Domingues et~al.(2021)Domingues, M{\'e}nard, Kaufmann, and
  Valko]{domingues2021episodic}
Domingues, O.~D., M{\'e}nard, P., Kaufmann, E., and Valko, M.
\newblock Episodic reinforcement learning in finite mdps: Minimax lower bounds
  revisited.
\newblock In \emph{Algorithmic Learning Theory}, pp.\  578--598. PMLR, 2021.

\bibitem[Garivier et~al.(2019)Garivier, M{\'e}nard, and
  Stoltz]{garivier2019explore}
Garivier, A., M{\'e}nard, P., and Stoltz, G.
\newblock Explore first, exploit next: The true shape of regret in bandit
  problems.
\newblock \emph{Mathematics of Operations Research}, 44\penalty0 (2):\penalty0
  377--399, 2019.

\bibitem[Jaksch et~al.(2010)Jaksch, Ortner, and Auer]{jaksch2010near}
Jaksch, T., Ortner, R., and Auer, P.
\newblock Near-optimal regret bounds for reinforcement learning.
\newblock \emph{Journal of Machine Learning Research}, 11\penalty0
  (Apr):\penalty0 1563--1600, 2010.

\bibitem[Jin et~al.(2020)Jin, Krishnamurthy, Simchowitz, and Yu]{jin2020reward}
Jin, C., Krishnamurthy, A., Simchowitz, M., and Yu, T.
\newblock Reward-free exploration for reinforcement learning.
\newblock \emph{arXiv preprint arXiv:2002.02794}, 2020.

\bibitem[Lim \& Auer(2012)Lim and Auer]{lim2012autonomous}
Lim, S.~H. and Auer, P.
\newblock Autonomous exploration for navigating in mdps.
\newblock In \emph{Conference on Learning Theory}, pp.\  40--1. JMLR Workshop
  and Conference Proceedings, 2012.

\bibitem[Mnih et~al.(2013)Mnih, Kavukcuoglu, Silver, Graves, Antonoglou,
  Wierstra, and Riedmiller]{mnih2013playing}
Mnih, V., Kavukcuoglu, K., Silver, D., Graves, A., Antonoglou, I., Wierstra,
  D., and Riedmiller, M.
\newblock Playing atari with deep reinforcement learning.
\newblock \emph{arXiv preprint arXiv:1312.5602}, 2013.

\bibitem[Oudeyer \& Kaplan(2009)Oudeyer and Kaplan]{oudeyer2009intrinsic}
Oudeyer, P.-Y. and Kaplan, F.
\newblock What is intrinsic motivation? a typology of computational approaches.
\newblock \emph{Frontiers in neurorobotics}, 1:\penalty0 6, 2009.

\bibitem[Oudeyer et~al.(2007)Oudeyer, Kaplan, and Hafner]{oudeyer2007intrinsic}
Oudeyer, P.-Y., Kaplan, F., and Hafner, V.~V.
\newblock Intrinsic motivation systems for autonomous mental development.
\newblock \emph{IEEE transactions on evolutionary computation}, 11\penalty0
  (2):\penalty0 265--286, 2007.

\bibitem[Schmidhuber(1991)]{schmidhuber1991possibility}
Schmidhuber, J.
\newblock A possibility for implementing curiosity and boredom in
  model-building neural controllers.
\newblock In \emph{Proc. of the international conference on simulation of
  adaptive behavior: From animals to animats}, pp.\  222--227, 1991.

\bibitem[Tarbouriech et~al.(2020)Tarbouriech, Pirotta, Valko, and
  Lazaric]{tarbouriech2020improved}
Tarbouriech, J., Pirotta, M., Valko, M., and Lazaric, A.
\newblock Improved sample complexity for incremental autonomous exploration in
  mdps.
\newblock \emph{arXiv preprint arXiv:2012.14755}, 2020.

\bibitem[Tarbouriech et~al.(2021)Tarbouriech, Zhou, Du, Pirotta, Valko, and
  Lazaric]{tarbouriech2021stochastic}
Tarbouriech, J., Zhou, R., Du, S.~S., Pirotta, M., Valko, M., and Lazaric, A.
\newblock Stochastic shortest path: Minimax, parameter-free and towards
  horizon-free regret.
\newblock \emph{arXiv preprint arXiv:2104.11186}, 2021.

\bibitem[Yu \& Bertsekas(2013)Yu and Bertsekas]{yu2013boundedness}
Yu, H. and Bertsekas, D.~P.
\newblock On boundedness of q-learning iterates for stochastic shortest path
  problems.
\newblock \emph{Mathematics of Operations Research}, 38\penalty0 (2):\penalty0
  209--227, 2013.

\end{thebibliography}

\newpage
\appendix
\onecolumn
\section{Basic Property of the Optimal Policy}
\begin{lemma}[\citealp{bertsekas1991analysis};\citealp{yu2013boundedness}] \label{lemma_wellposedproblem}
Suppose that there exists a proper policy with respect to the goal state $g$ and that for every improper policy $\pi'$ there exists at least one state $s \in \cS$ such that $V_g^{\pi'}(s) = + \infty$. Then the optimal policy $\pi^{{*}}$ is stationary, deterministic, and proper. Moreover, $V_g^{*} = V_g^{\pi^{*}}$ is the unique solution of the optimality equations $V_g^{*} = \mathcal{L} V_g^{*}$ and $V_g^{*}(s) < +\infty$ for any $s \in \mathcal{S}$, where for any vector $V \in \mathbb{R}^S$ the optimal Bellman operator $\mathcal{L}$ is defined as
\begin{align*}
\mathcal{L}V(s) := \min_{a \in \mathcal{A}} \Big\{ c(s, a) + P_{s,a} V \Big\}. 
\end{align*}
Furthermore, the optimal $Q$-value, denoted by $Q_g^{{*}}  = Q_g^{\pi^{{*}}}$, is related to the optimal value function as follows
    \begin{align*}
    Q_g^{{*}}(s,a) = c(s,a) + P_{s,a} V_g^{{*}}, \quad \quad V_g^{{*}}(s) = \min_{a \in \cA} Q_g^{{*}}(s,a), \quad \quad \forall (s,a) \in \cS \times \cA. 
\end{align*}
\end{lemma}


\section{High-Probability Event}
\label{appendix_high_prob}


First we define the high-probability event $\calE$ to do concentration on all the samples collected in Alg.\,\ref{algo2} and Alg.\,\ref{algo3}. 
We note that in Alg.\,\ref{algo3}, after running DisCo algorithm, the set of known states $\cK$ is fixed, and our algorithm focuses on the new MDP $\cM^\dagger = \langle \cK^\dagger, \cA, P^\dagger, c^\dagger, s_0 \rangle$, where $\cM^\dagger$ is defined in Sect.\,\ref{sect_algo2}.

We recall that for any two vectors $X,Y \in \mathbb{R}^{K'}$ ($K' = |\cK^\dagger|$), we write their inner product as $XY := \sum_{s \in \cK^\dagger} X(s)Y(s)$, and we denote $\|X\|_\infty := \max_{s \in \cK^\dagger} |X(s)|$. If $X$ is a probability distribution on $\cK^\dagger$, we denote $\mathbb{V}(X, Y) := \sum_{s \in \cK^\dagger} X(s) Y(s)^2 - (\sum_{s \in \cK^\dagger} X(s)Y(s))^2$, i.e. the variance of random varianble $Y$ over distribution $X$. And we use $P^\dagger_{s,a}$ and $P^\dagger_{s,a,s'}$ to denote $P^\dagger(\cdot \vert s,a)$ and $P^\dagger(s' \vert s,a)$,  respectively.

Here for any $g \in \cK$, we denote the vector $V^*_{g} \in \mathbb{R}^{K'}$ as the value function of the optimal policy on MDP $\cM^\dagger$ with respect to goal $g$, and we denote $V_g^*(s)$ as the expected cost of the optimal policy to reach state $g$ from $s$ on MDP $M^\dagger$.
Also, we define $B_*:= \max\limits_{(s,g) \in \cK^\dagger\times\cK} V^*_g(s)$, and we denote $Q^*_g \in \mathbb{R}^{K' \times A}$ as the $Q$-function corresponding to $V^*_{g}$.  

In Alg.\,\ref{algo3}, we set $B = 10L$. Thus when $\cK \subseteq \cS_{2L}^\rightarrow$, we have $\forall (s,g) \in \cK^\dagger\times\cK, V_g^*(s) \leq 2L+1 \leq B$, and $B_* \leq B$.

Then we define the high-probability event $\calE$. We note that our definition of $\calE$ is similar with Definition 12 in Sect. D.1,  \citep{tarbouriech2021stochastic}.

\begin{definition}[High-probability event $\calE$]\label{highprobevent}
We define the event $\calE := \calE_1 \cap \calE_2 \cap \calE_3$, where
\begin{align*}
    \calE_1 &:= \left\{ \forall (s,a) \in \cK^\dagger \times \cA, \forall n(s,a) \geq 1 \,: ~ \vert \wh{c}(s,a) - c^\dagger(s,a) \vert \leq 2 \sqrt{\frac{ 2 \wh{c}(s,a) \iota_{s,a}}{ n(s,a) }} + \frac{28 \iota_{s,a}}{ 3 n(s,a)} \right\}, \\
    \calE_2 &:= \left\{ \forall (s,a,s') \in \cK^\dagger \times \cA \times \cK^\dagger, ~ \forall n(s,a) \geq 1 \,: ~  |P^\dagger_{s,a,s'}-\wh P_{s,a,s'}| \leq \sqrt{\frac{2P^\dagger_{s,a,s'} \iota_{s,a}}{n(s,a)}}+\frac{\iota_{s,a}}{n(s,a)} \right\}, \\
    \calE_3 &:= \left\{\forall (s,a,g) \in \cK^\dagger \times \cA \times \cK, \forall n(s,a) \geq 1 \, : ~  \vert (\wh P_{s,a} - P^\dagger_{s,a}) V_g^{{*}} \vert \leq 2 \sqrt{\frac{ \mathbb{V}(\wh P_{s,a}, V_g^{{*}}) \iota_{s,a}}{ n(s,a) }} + \frac{14 B_* \iota_{s,a}}{ 3 n(s,a)}   \right\}, 
\end{align*}
where $\iota_{s,a} := 4\ln \left( \frac{12 K' A n(s,a)}{\delta} \right)$. 
\end{definition}
\begin{lemma}
It holds that $\mathbb{P}(\calE) \geq 1 - \delta$.
\end{lemma}
\begin{proof}
The proof is the same with Lemma 13 in Sect. D.1, \citep{tarbouriech2021stochastic}.

The event $\calE_1$ holds with probability $ 1 - \delta / 3$ by Lem. 27 in Sect. F, \citep{tarbouriech2021stochastic} and by union bound over all $(s,a) \in \cK^\dagger \times \cA$.

The event $\calE_2$ holds with probability $1 - \delta / 3$ by Lem. 26 and Lem. 33 in Sect. F, \citep{tarbouriech2021stochastic} and by union bound over $(s,a,s') \in \cK^\dagger \times \cA \times \cK^\dagger$.

The event $\calE_3$ holds with probability $ 1 - \delta / 3$ by Lem. 27 in Sect. F, \citep{tarbouriech2021stochastic} and by union bound over all $(s,a,g) \in \cK^\dagger \times \cA \times \cK$.
\end{proof}

The lemma above is a direct consequence of concentration inequalities. We note that we do not use the samples collected in DisCo algorithm, and the set $\cK$ is fixed after running DisCo algorithm. Hence for any $g \in \cK$, the vector $V_g^*$ depends only on the set $\cK$, and $V_g^*$ is fixed after running DisCo algorithm and does not depend on the samples collected in Alg.\ref{algo2} and Alg.\ref{algo3}. Thus $\wh P_{s,a}$ and $V_g^*$ are independent for any $(s,a,g) \in \cK^\dagger \times \cA \times \cK$ and $n(s,a) \geq 1$.   

\newcommand{\exptail}
{Let $\pi$ be a proper policy such that for some $d>0$, the expected cost $V^{\pi}_g(s) \leq d$ for every non-goal state $s \neq g$. Then the probability that the cumulative cost of $\pi$ to reach the goal state from any state s is more than $m$, is at most $2 e^{-m /(4 d)}$ for all $m \geq 0$.}
\begin{lemma}[\cite{cohen2020near}, Lem.\,B.5]
\label{exptail}
\exptail
\end{lemma}

\newcommand{\probbound}{
Let $\tau$ be a random variable on $[0,+\infty)$ such that $\Pr(\tau > m) \leq 2e^{-m/4d}$ for any $m \geq 0$, where $d > 0$ is a constant. We define the random variable $\hat{\tau} = \frac{1}{n}\sum\limits_{k=1}^n \hat{\tau}_k$, where each $\hat{\tau}_k$ is i.i.d. and has the same distribution with $\tau$. Then for any $\epsilon > 0$, we have $ \Pr(E(\tau) > \hat{\tau} + \epsilon d) \leq \exp(-\frac{n \epsilon^2}{128\ln^2(64/\epsilon)})$.
}
\begin{lemma}
\label{probbound}
\probbound
\end{lemma}
\begin{proof}
We set the constant $ \Gamma = \lfloor 8d\ln(64/\epsilon) \rfloor$. Then we define the random variables $\tau_\Gamma = \min(\tau, \Gamma)$, $\check{\tau}_k = \min(\hat{\tau}_k, \Gamma)$, and $\check{\tau} = \frac{1}{n}\sum\limits_{k=1}^n \check{\tau}_k.$

As each $\check{\tau}_k$ is a random variable on $[0,\Gamma]$, by Hoeffding's inequality, we have
$$
\Pr(E(\tau_\Gamma) > \check{\tau} + \frac{1}{2}\epsilon d) \leq \exp(-\frac{n \epsilon^2 d^2}{2\Gamma^2}) \leq \exp(-\frac{n \epsilon^2}{128\ln^2(64/\epsilon)}).
$$

Moreover, we have 
\begin{align*}
    E(\tau) &\leq E(\tau_\Gamma) + \sum\limits_{i=1}^\infty i \cdot \Pr(\Gamma + i - 1 < \tau \leq \Gamma + i) = E(\tau_\Gamma) + \sum\limits_{m=\Gamma}^\infty \Pr(\tau > m) \nonumber \\
    &\leq  E(\tau_\Gamma) + 2\sum\limits_{m=\Gamma}^\infty \exp(-m/4d) 
    \leq  E(\tau_\Gamma) + \frac{1}{2}\epsilon d. \nonumber
\end{align*}

Therefore, we obtain
$$
\Pr(E(\tau) > \hat{\tau} + \epsilon d) \leq \Pr(E(\tau) > \check{\tau} + \epsilon d) \leq \Pr(E(\tau_\Gamma) > \check{\tau} + \frac{1}{2}\epsilon d) 
\leq \exp(-\frac{n \epsilon^2}{128\ln^2(64/\epsilon)}).
$$

\end{proof}

\section{Analysis of a \VISGO Procedure}
\label{appendix_VISGO}

\begin{algorithm}
        \small
\begin{algorithmic}[1]
        \STATE \textbf{Input:} Goal state $g$ and $\epsVI$. \\
        \STATE \textbf{Global variables:} $ B,\ L,\ N(),\ n(), \ \wh P, \ \theta(), \ \wh{c}()$. \\
        \STATE \textbf{Specify:} Constants $c_1 = 6,\ c_2 = 72,\ c_3 = 2 \sqrt{2}$. \\
       \STATE For all $(s,a,s') \in \mathcal{K} \times \cA \times \mathcal{K^\dagger}$, set
        \begin{align*} 
            \wt P_{s,a,s'} &\leftarrow \frac{n(s,a)}{n(s,a)+1} \wh P_{s,a,s'} + \frac{\mathds{I}[s'=g]}{n(s,a)+1}.
        \end{align*}\\
        \STATE For all $(s,a) \in \cK \times \cA$, set  $ \iota_{s,a} \gets 4\ln \left( \frac{12 K' A n (s,a)}{\delta} \right)$.\\
        \STATE Set $i \leftarrow 0$, $V^{(0)} \leftarrow 0$, $V^{(-1)} \leftarrow +\infty$. \\
        \WHILE{$\norm{ V^{(i)} - V^{(i-1)}}_{\infty} > \epsVI$}
        
        \STATE For all $s \in \cK\setminus\{g\}$ and $a\in\cA$, set 
        \begin{align}
            &b^{(i+1)}(s,a) \, \leftarrow \, \max \Big\{ c_1 \sqrt{ \frac{\mathbb{V}(\wh P_{s,a}, V^{(i)}) \iota_{s,a} }{n(s,a)}} , \, c_2 \frac{ B \iota_{s,a}}{n(s,a)}  \Big\} + c_3 \sqrt{\frac{  \wh{c}(s,a) \iota_{s,a}}{n(s,a)}}, \\
            &Q^{(i+1)}(s,a) \, \leftarrow \,  \wh{c}(s,a) \,+ \, \wt P_{s,a} V^{(i)}  ~ - \,b^{(i+1)}(s,a), \label{update_Q}\\
            &V^{(i+1)}(s) \, \leftarrow \, \min_{a\in\cA} Q^{(i+1)}(s,a).
        \end{align}\\
        \STATE Set $V^{(i+1)}(x)  \leftarrow \cRESET + V^{(i)}(s_0)$.
        \\
        \STATE Set $V^{(i+1)}(g) \leftarrow 0$, \ $i \leftarrow i+1$.\\
        \ENDWHILE
        \STATE \textbf{return} $Q^{(i)},\ V^{(i)}$.
        \caption{Subroutine \VISGO}
        \label{plan}
\end{algorithmic}
\end{algorithm}

In this section, we fix the known states $\cK$ and the goal state $g$ and we analysis an execution of the  \VISGO procedure in Alg.\,\ref{plan}. We use the value iteration of the form $ V^{(i+1)} = \wtcalL V^{(i)}$ to estimate the value funtion of the optimal policy. 
Here, we define the operator $\wtcalL$ in the following way. For any $U \in \mathbb{R}^{K'}$ such that $U \geq 0$, $U(g) = 0$, and $\|U\|_\infty \leq B$, we first define 
\begin{align*} 
    \wtcalL U(s,a) :=  \wh{c}(s,a) + \wt P_{s,a} U - b(U,s,a),
\end{align*}%
for any $s \in \cK \setminus \{g\}$ and $a  \in \cA$, where we define 
\begin{eqnarray}\label{bonus}
b(U,s,a) := \max \left\{ c_1 \sqrt{ \frac{\mathbb{V}(\wh P_{s,a}, U) \iota_{s,a} }{n(s,a)}} , \ c_2 \frac{B \iota_{s,a}}{n(s,a)}  \right\} + c_3 \sqrt{\frac{\wh{c}(s,a) \iota_{s,a}}{n(s,a)}}, 
\end{eqnarray}
for any $s \in \cK \setminus \{g\}$ and $a  \in \cA$. Here we recall that $B = 10L$,  $\iota_{s,a} = 4\ln \left( \frac{12 K' A n(s,a)}{\delta} \right)$ (cf. Def.~\ref{highprobevent}), and $\mathbb{V}(X, Y) := \sum_{s \in \cK^\dagger} X(s) Y(s)^2 - (\sum_{s \in \cK^\dagger} X(s)Y(s))^2$ is the variance of random varianble $Y$ over distribution $X$. And we define the transition probability $\wt P_{s,a,s'} =  \frac{n(s,a)}{n(s,a)+1} \wh P_{s,a,s'} + \frac{\mathds{I}[s'=g]}{n(s,a)+1}$ that slightly increases the probability to reach the goal $g$ at each state-action pair. 

Then, we set $\wtcalL U(s) :=  \min_{a \in \cA} \wtcalL U(s,a)$ for $s\in \cK$ and $s \neq g$, and we set $\wtcalL U(x) := \cRESET +  U(s_0)$. Finally, we set $\wtcalL U(g) := 0$. 

We note that in Alg.\,\ref{algo3}, before we executed \VISGO procedure, we have collected $\phi = \wt{O}(L^2 |\cK| / \cmin^2)$ samples for each state-action pair $(s,a) \in \cK \times \cA$ in Alg.\,\ref{algo2}. Thus we have $n(s,a) \geq \phi$ for each $(s,a) \in \cK \times \cA$. We stress that the lemmas of this section are based on the conditions that $n(s,a) \geq \phi$ for all $(s,a) \in \cK \times \cA$.

We note that the variance $\mathbb{V}(\wh P_{s,a}, U) \leq O(L^2)$ when $\|U\|_{\infty} \leq B = 10L$, and $\iota_{s,a}$ contains only logarithmic terms, thus $b(U,s,a) = \wt{O}(L / \sqrt{n(s,a))}$. As $n(s,a) \geq \phi = \Omega(L^2K\cmin^{-2})$, we have $b(U,s,a) \leq \cmin/18 \leq \cmin$ for any $(s,a) \in \cK\times\cA$. Therefore, if $U(g) = 0$, $\|U\|_{\infty} \leq B$, and $U \geq 0$ component-wise, when $n(s,a) \geq \phi$ for any $(s,a) \in \cK \times \cA$, we have $\wtcalL U(s,a) \geq 0$ for any $(s,a) \in \cK\times\cA$. Hence the output of \VISGO $(Q,V)$ satisfies $Q \geq 0$ and $V \geq 0$ component-wise.

For convenience, we define $b(U,x,a) := 0$ and $b(U,g,a) := 0$ for any $a \in \cA$.

\newcommand{\lemmapropertieswtP}{
    For any non-negative vector $U \in \mathbb{R}^{K'}$ such that $U(g) = 0$, for any $(s,a) \in \cK \times \cA$, it holds that
    \begin{align*}
        \wt P_{s,a} U \leq \wh P_{s,a} U \leq \wt P_{s,a} U + \frac{\norm{U}_{\infty}}{n(s,a) + 1}.
    \end{align*}
}
\begin{lemma}[\cite{tarbouriech2021stochastic}, Lemma 12]
  \label{lemma_properties_wtP}
  \lemmapropertieswtP
\end{lemma}

The proof of the following  Lem.\,\ref{lemma_properties_f} is similar with Lem.\,16 in \citep{tarbouriech2021stochastic}, but here we have two distributions $\tilde{p}$ and $p$. We give the whole proof for completeness. 

\newcommand{\lemmapropertiesf}{
    Let $\Upsilon := \{ v \in \mathbb{R}^{K'}: v \geq 0, ~v(g) = 0, ~ \norm{v}_{\infty} \leq B \}$. Let $f: \Delta^{K'} \times \Delta^{K'} \times \Upsilon \times \mathbb{R} \times \mathbb{R} \times \mathbb{R} \rightarrow \mathbb{R}$ with $f(\tilde{p},p,v,n,B, \iota) := \tilde{p} v - \max\Big\{ c_1 \sqrt{\frac{\mathbb{V}(p,v) \iota}{n}}, \, c_2 \frac{B \iota }{n} \Big\}$, with constants $c_1 = 6$ and $c_2 \geq 2 c_1^2$. Then $f$ satisfies, for all $v \in \Upsilon$, $n, \iota > 0$, $\tilde{p}, p \in \Delta^{K'}$ s.t. $\tilde{p}(s) - \frac{1}{2} p(s) \geq 0$ for all $s \neq g$,
\begin{enumerate}
    \item $f(\tilde{p},p,v,n,B, \iota)$ is non-decreasing in $v(s)$, i.e.
    \begin{align*}
        \forall (v, v') \in \Upsilon^2, ~ v \leq v' ~ \implies ~ f(\tilde{p},p,v,n,B, \iota) \leq f(\tilde{p},p,v',n,B, \iota);
    \end{align*}
    \item $f(\tilde{p},p,v,n,B, \iota) \leq \tilde{p} v - \frac{c_1}{2} \sqrt{\frac{\mathbb{V}(p,v) \iota}{n}} - \frac{c_2}{2} \frac{ B \iota}{n} \leq \tilde{p} v - 2 \sqrt{\frac{\mathbb{V}(p,v) \iota}{n}} - 14 \frac{ B \iota}{n}$; 
    \item If $\tilde{p}(g) > 0$, then $f(\tilde{p},p,v,n,B, \iota)$ is $\rho_{\tilde{p}}$-contractive in $v(s)$, 
    with $\rho_{\tilde{p}} := 1 - p(g) < 1$, i.e.
    \begin{align*}
        \forall (v, v') \in \Upsilon^2, ~ \abs{ f(\tilde{p},p,v,n,B, \iota) - f(\tilde{p},p,v',n,B, \iota) } \leq \rho_{\tilde{p}} \norm{ v - v'}_{\infty}.
    \end{align*}
\end{enumerate}
}
\begin{lemma}
  \label{lemma_properties_f}
  \lemmapropertiesf
\end{lemma}
\begin{proof}
We use the idea in \cite{tarbouriech2021stochastic}, Lemma 14 to finish the proof.

The second claim holds by $\max \{x, y \} \geq (x+y)/2, \forall x, y$, by the choices of $c_1, c_2$ and because both $\sqrt{\frac{\mathbb{V}(p,v) \iota}{n}}$ and $\frac{B \iota}{n}$ are non-negative.
To verify the first and third claims, we fix all other variables but $v(s)$ and view $f$ as a function in $v(s)$. Because the derivative of $f$ in $v(s)$ does not exist only when $c_1 \sqrt{\frac{\mathbb{V}(p,v) \iota}{n}} = c_2 \frac{B \iota }{n}$, where the condition has at most two solutions, it suffices to prove $\diffp{f}{{v(s)}} \geq 0$ when $c_1 \sqrt{\frac{\mathbb{V}(p,v) \iota}{n}} \neq c_2 \frac{B \iota }{n}$. Direct computation gives that for any $s \in \cK^\dagger$ and $s \neq g$,
\begin{align*}
    \diffp{f}{{v(s)}} &= \tilde{p}(s) - c_1 \mathds{I}\left[  c_1 \sqrt{ \frac{\mathbb{V}(p,v) \iota}{n} } \geq c_2 \frac{B \iota}{n} \right] \frac{p(s)(v(s) - pv) \iota}{ \sqrt{n \mathbb{V}(p,v) \iota } } \\
    &\geq \min \big\{ \tilde{p}(s) , ~ \tilde{p}(s) - \frac{c_1^2}{c_2 B} p(s) \big( v(s) - p v \big) \big\} \\
    &\mygei  \min \big\{ \tilde{p}(s) , ~ \tilde{p}(s) - \frac{c_1^2}{c_2 } p(s) \big\} \\
    &\geq 0.
\end{align*}

Here (i) is by $v(s)-p v \leq v(s) \leq B$. In addition, we have
\begin{align*}
    \sum\limits_{s\neq g} \left|\diffp{f}{{v(s)}}\right| &=\sum\limits_{s\neq g}\left[ \tilde{p}(s) - c_1 \mathds{I}\left[  c_1 \sqrt{ \frac{\mathbb{V}(p,v) \iota}{n} } \geq c_2 \frac{B \iota}{n} \right] \frac{p(s)(v(s) - pv) \iota}{ \sqrt{n \mathbb{V}(p,v) \iota } } \right]\\
    &= 1 - \tilde{p}(g) - 
    c_{1} \mathds{I}\left[c_{1} \sqrt{\frac{\mathbb{V}(p, v) \iota}{n}} \geq c_{2} \frac{B \iota}{n}\right] \sqrt{\frac{\iota}{n \mathbb{V}(p, v)}}[p v-(1-p(g)) \cdot p v] \\
    &\leq 1 - \tilde{p}(g).
\end{align*}
Therefore, we obtain that $f$ is $\rho_{\tilde{p}}$-contractive.
\end{proof}

We note that by definition of $\wt{P}_{s,a}$, we have $ \wt{P}_{s,a,s'} - \frac{1}{2}\wh P_{s,a,s'} \geq 0$ for all $(s,a,s') \in \cK \times \cA \times \cK^\dagger$.

The following two lemmas follow the same proof with Lem.18, Lem.19 in  \citep{tarbouriech2021stochastic}, respectively.

\newcommand{\lemmawtcalLconvergent}{
The sequence $(V^{(i)})_{i \geq 0}$ is non-decreasing.
}
\begin{lemma}
  \label{lemma_wtcalL_convergent}
  \lemmawtcalLconvergent
\end{lemma}

\newcommand{\lemmawtcalLcontraction}{
$\wtcalL$ is a $\rho$-contractive operator with modulus $\rho := 1 - \nu < 1$, where $\nu = \min\limits_{(s,a)\in\cK\times\cA} \wt{P}_{s,a,g}$, i.e. for any two vectors $U_1, U_2 \in \Upsilon$ (where $\Upsilon$ is defined in Lem.\,\ref{lemma_properties_f}), $\|\wtcalL U_1 - \wtcalL U_2 \|_\infty \leq \rho \| U_1 - U_2\|_\infty$. Hence, the \VISGO procedure will terminate after at most $\lceil \log(1/\epsVI) / \log(1/\rho) \rceil$ iterations.
}
\begin{lemma}
  \label{lemma_wtcalL_contraction}
  \lemmawtcalLcontraction
\end{lemma}

\subsection{The Bounded Error Property of \VISGO}\label{app_bounded_error}
Now we focus on Alg.\,\ref{algo3}.
We give the following lemma of the bounded error property (Lem.\,\ref{consterror}), which indicates that the value function of the policy $\pi_s$ is close to our estimation. 
The proof of Lem.\,\ref{consterror} uses the techniques of Lem. 2 in  \citep{tarbouriech2020improved}. Our Lem.\,\ref{consterror} focuses on a more general operator $\wtcalL$. In our $\wtcalL$, we involve the bonus function $b(U,s,a)$, which is not contained in \citep{tarbouriech2020improved}. And we note that in our proof of the following Lem.~\ref{consterror}, we use the condition that $n(s,a) \geq \phi = \wt{\Omega}(L^2 K \cmin^{-2})$ for each $(s,a) \in \cK \times \cA$. Also, we have $\epsVI \leq \cmin/18$ because we set the initial trigger index $j = 5 + \log_2 \cmin^{-1}$ and $\epsVI = 2^{-j}$. 

We note that by optimism property (Lem.~\ref{lemma_opt}), when $\cK \subseteq \cS_{2L}^\rightarrow$, we have $V(s) \leq 2L+1$ for all $s \in \cK^\dagger$. Hence the following bounded error property (Lem.~\ref{consterror}) implies that in any round, the expected cost of the greedy policy $\tilde{\pi}$ on model $P^\dagger$ is no more than $2(2L+1) = O(L)$.
\newcommand{\consterror}{In Alg.\,\ref{algo3}, under the event $\calE$, for any output $(Q,V)$ of the \VISGO procedure in any round, let $\pi$ be the greedy policy with respect to $Q$. Then $\pi$ is proper on the model $P^\dagger_{s,a,s'}$, and for all $s\in \cK^\dagger$, we have $ V^{\pi}_g(s) \leq 2 V(s)$, where $g$ is the goal state in that round.}
\begin{lemma}[Bounded Error Property]
\label{consterror}
\consterror
\end{lemma}
\begin{proof}
We define $\wt{V}^\pi_g(s)$ as the value function of $\pi$ with goal state $g$ on the model $\wt{P}_{s,a,s'}$. We will first prove that $\wt{V}^\pi_g(s) \leq \frac{4}{3} V(s)$, and then prove that $ V^\pi_g(s) \leq \frac{4}{3} \wt{V}^\pi_g(s)$ using the simulation lemma on the two models $\wt{P}_{s,a,s'}$ and $P^\dagger_{s,a,s'}$. Combining them together yields $ V^{\pi}_g(s) \leq 2 V(s)$.

First we focus on model $\wt{P}_{s,a,s'}$.  We recall that for any $s \in \cK$ and $s \neq g$,
$$\wtcalL u(s):=\min _{a \in \mathcal{A}}\left\{\wh{c}(s,a)-b(u,s,a)+\widetilde{P}_{s,a}u \right\}$$
where $b(u,s,a)$ is defined in Eq.~\ref{bonus}, i.e., for any $s \in \cK \setminus \{g\}$ and $a \in \cA$,
$$b(u,s,a) = \max \left\{c_{1} \sqrt{\frac{\mathbb{V}\left(\wh{P}_{s, a}, u\right) \iota_{s, a}}{n(s, a)}}, c_{2} \frac{B \iota_{s, a}}{n(s, a)}\right\}+c_{3} \sqrt{\frac{\iota_{s, a}}{n(s, a)}},$$
and we define $b(u,x,a) = 0$ and $b(u,g,a) = 0$.

We observe that when $\|u \|_\infty \leq B = 10L$, the variance $\mathbb{V}\left(\wh{P}_{s, a}, u\right) \leq B^2$, and $\iota_{s,a}$ contains only logarithmic terms. Thus we have $b(u,s,a) = \wt{O}(L / \sqrt{ n(s,a)})$.

As we set $$\phi = \Theta(\frac{L^2|\cK|}{\cmin^2}\ln(\frac{|\cK|A}{\delta})),$$ and $n(s,a) \geq \phi$, we can obtain $b(u,s,a) \leq {\cmin}/{18}$ when $\|u\|_\infty \leq B$.

In addition, under the event $\calE_1$, we have $\vert \wh{c}(s,a) - c(s,a) \vert \leq \wt{O}(1 / \sqrt{n(s,a)})$.

Thus when $n(s,a) \geq \phi$, we have $\vert \wh{c}(s,a) - c(s,a) \vert \leq \cmin/18$ for all $(s,a) \in \cK^\dagger \times \cA$.

We denote $l$ as the final iteration index of \VISGO, and $V = V^{(l)}$. In $\texttt{VISGO}$, we have $V^{(i)} = \wtcalL V^{(i-1)}$ for all $i=1,2,\cdots,l$. As $V^{(l-1)} \leq V^{(l)}$ component-wise, we have for any $s\in \cK^\dagger$, $V(s) \leq   V^{(l-1)}(s_0) + 1 \leq 2L+1$. Thus, $\|V\|_\infty \leq 2L+1 \leq B$.

We set $\gamma = {\cmin}/{6}$. As $\epsVI \leq \cmin/18$, we have $b(u,s,a) + \vert \wh{c}(s,a) - c(s,a) \vert + \epsVI \leq \gamma$ when $\|u\|_\infty \leq B$.

Given the policy $\pi$ restricted on $\cK$,  we introduce the following operators on $\mathbb{R}^{K'}$:
$$\mathcal{L}^{\pi} u(s) = \wh{c}(s,\pi(s))-b(u,s,\pi(s))+\widetilde{P}_{s,\pi(s)}u,$$
$$
\mathcal{T}_{\gamma}^{\pi} u(s):={  c(s,\pi(s))-\gamma}+\widetilde{P}_{s,\pi(s)}u.
$$

We can write component-wise 
$$
\mathcal{T}_{\gamma}^{\pi} V \leq \mathcal{L}^{\pi} V-\epsVI \stackrel{(\mathrm{a})}{=} \wtcalL V-\epsVI \stackrel{\mathrm{(b)}} \leq V,
$$
where $(a)$ uses that $\pi$ is the greedy policy with respect to $V$. To prove (b), we recall that $V = V^{(l)} = \wtcalL V^{(l-1)}$. By contraction property of $\wtcalL$ (Lem.~\ref{lemma_wtcalL_contraction}), we have $\|\wtcalL V - V \|_\infty \leq \| V^{(l)} - V^{(l-1)} \|_\infty$. By stopping condition of \VISGO, we have $ \| V^{(l)} - V^{(l-1)} \|_\infty \leq \epsVI$, thus (b) is proved. By monotonicity of the Bellman operator $\mathcal{T}_{\gamma}^{\pi}$, we have for all $m > 0$, $ (\mathcal{T}_{\gamma}^{\pi})^m V \leq (\mathcal{T}_{\gamma}^{\pi})^{m-1} V \leq \cdots \leq V$.

We observe that the vector $(\mathcal{T}_{\gamma}^{\pi})^m V$ does not increase element-wise when $m$ increases, and $(\mathcal{T}_{\gamma}^{\pi})^m V \geq 0$ element-wise because $V \geq 0$ element-wise. Hence when $m \rightarrow \infty$, it will converge to some vector $W_{\gamma}^{\pi}$, where $W_{\gamma}^{\pi}$ is the value function of policy $\pi$ in the model $\widetilde{P}$ with $\gamma$ subtracted to all the costs, and we have $W_{\gamma}^{\pi} \leq V$ component-wise. We define the random variable $\tilde{t}^{\pi}_g(s)$ as the number of steps it takes to reach $g$ starting from $s$ on model $\wt{P}$ when executing policy $\pi$. Thus
$$W_{\gamma}^{\pi}(s):=\mathbb{E}_{\wt{P}}\left[\sum_{t=1}^{\tilde{t}^{\pi}_g(s)}c\left(s_{t}, \pi\left(s_{t}\right)\right)-\gamma \mid s_{1}=s\right]=\wt{V}^{\pi}_g(s)-\gamma \mathbb{E}_{\wt{P}}\left[\tilde{t}^{\pi}_g(s)\right].$$
Moreover, we have $\cmin\mathbb{E}\left[\tilde{t}^{\pi}_g(s)\right] \leq \wt{V}^{\pi}_g(s)$.
Therefore, we get $$\wt{V}^{\pi}_g(s) \leq \frac{W_{\gamma}^{\pi}(s)}{1 - \gamma/\cmin} \leq \frac{V(s)}{1 - \gamma/\cmin} \leq \frac{4}{3} V(s).$$

Under the event $\calE_2$, we have $\left|P^\dagger_{s, a, s^{\prime}}-\wh{P}_{s, a, s^{\prime}}\right| = \wt{O}(\sqrt{P^\dagger_{s, a, s^{\prime}} / n(s,a)} + n(s,a)^{-1})$. As $n(s,a) \geq \phi = \wt{\Omega}(L^2|\cK|)$, and $B = 10L$, we can obtain $\forall\left(s, a, s^{\prime}\right) \in \cK^\dagger \times \mathcal{A} \times \cK^\dagger$, $$\left|P^\dagger_{s, a, s^{\prime}}-\wh{P}_{s, a, s^{\prime}}\right| \leq \frac{\cmin }{24B} \sqrt{\frac{P^\dagger_{s, a, s^{\prime}}}{|\cK^\dagger|}} + \frac{\cmin}{24B |\cK^\dagger|}.$$

By the Cauchy–Schwarz's inequality,  $\sum\limits_{s' \in \cK^\dagger} \sqrt{P^\dagger_{s, a, s^{\prime}}} \leq \sqrt{|\cK^\dagger|}$, hence we obtain
$$\sum\limits_{s' \in \cK^\dagger} \left|P^\dagger_{s, a, s^{\prime}}-\wh{P}_{s, a, s^{\prime}}\right| \leq \frac{\cmin }{12B}, \quad \forall\left(s, a\right) \in \cK^\dagger \times \mathcal{A}.$$

Also, as $|\wt{P}_{s, a, s^{\prime}}-\wh{P}_{s, a, s^{\prime}}| \leq {1}/{n(s,a)} \leq {\cmin }/{(12B|\cK^\dagger|)}$, and $ \wt{V}^{\pi}_g(s) \leq \frac{4}{3}(2L+1) \leq B$ for all $s \in \cK^\dagger$, we can obtain that 
$$\sum\limits_{s'\in\cK^\dagger} \left|P^\dagger_{s, a, s^{\prime}}-\wt{P}_{s, a, s^{\prime}}\right| \leq \frac{\cmin }{6 \|\wt{V}^{\pi}_g\|_{\infty}}, \quad \forall\left(s, a\right) \in \cK^\dagger \times \mathcal{A}.$$

Thus by simulation lemma for SSP (Lemma 3 in \cite{tarbouriech2020improved}), $\pi$ is proper on true model $P^\dagger_{s,a,s'}$, and for all $s \in \cK^\dagger$, $V^{\pi}_g(s) \leq (1 + \frac{1}{3})\wt{V}^{\pi}_g(s) = \frac{4}{3} \wt{V}^{\pi}_g(s) \leq 2V(s)$. The proof is completed.
\end{proof}

\subsection{Optimistic Property of \VISGO}\label{app_optimism}
Now we will give the optimistic property. We still focus on Alg.\,\ref{algo3}, and we will prove that the output of the \VISGO procedure $(Q,V)$ is optimistic. And we recall that we denote $V^*_g$ as the value function of the optimal policy on MDP $\cM^\dagger$ to reach $g$, and $Q^*_g$ as the $Q$-function corresponding to $V^*_g$.

\newcommand{\lemmaopt}{ 
In Alg.\ref{algo3}, 
under the event $\calE$, for any output $(Q,V)$ of the \VISGO procedure, it holds that
\begin{align*}
    Q(s,a) &\leq Q_{g}^{{*}}(s,a), &\forall s \in \cK \setminus \{g\}, a \in \cA, \\
    V(s) &\leq V_{g}^*(s), &\forall s \in \cK^\dagger,
\end{align*}
where $g$ is the goal state in \VISGO procedure.
}
\begin{lemma}[Optimistic Property]
  \label{lemma_opt}
  \lemmaopt
\end{lemma}
\begin{proof}
We prove by induction that for any inner iteration $i$ of \VISGO, $Q^{(i)}(s,a) \leq Q^{{*}}_{g}(s,a)$ for any $(s,a) \in \cK \times \cA$, and $ V^{(i)}(s) \leq V_{g}^*(s)$ for any $s \in \cK^\dagger$. By definition we have $Q^{(0)} = 0 \leq Q^{{*}}_{g}$, and $V^{(0)} = 0 \leq V_{g}^*$. Assume that the optimistic property holds for iteration $i$, then for any $(s,a) \in \cK \times \cA$ and $s \neq g$,
\begin{align*}
    Q^{(i+1)}(s,a) &=  \wh{c}(s,a) + \wt P_{s,a} V^{(i)} - b^{(i+1)}(s,a),
\end{align*}
where
\begin{align*}
    &\wh{c}(s,a) + \wt P_{s,a} V^{(i)} - b^{(i+1)}(s,a) \\
    &= \wh{c}(s,a) + \wt P_{s,a} V^{(i)} - \max \Big\{ c_1 \sqrt{ \frac{\mathbb{V}(\wh P_{s,a}, V^{(i)}) \iota_{s,a} }{n(s,a)}}, \, c_2 \frac{B \iota_{s,a}}{n(s,a)}  \Big\} - c_3 \sqrt{\frac{\wh{c}(s,a) \iota_{s,a}}{n(s,a)}}  \\
    &\myineqi c(s,a) + \wt P_{s,a} V^{(i)} - \max \Big\{ c_1 \sqrt{ \frac{\mathbb{V}(\wh P_{s,a}, V^{(i)}) \iota_{s,a} }{n(s,a)}}, \, c_2 \frac{B \iota_{s,a}}{n(s,a)}  \Big\} + \frac{28 \iota_{s,a}}{3 n(s,a)}  \\
    &= c(s,a) + f( \wt P_{s,a}, \wh P_{s,a}, V^{(i)}, n(s,a), B, \iota_{s,a}) + \frac{28 \iota_{s,a}}{3 n(s,a)} \\
    &\myineqii c(s,a) + f( \wt P_{s,a},\wh P_{s,a}, V^{{*}}_{g}, n(s,a), B, \iota_{s,a}) + \frac{28 \iota_{s,a}}{3 n(s,a)}  \\
    &\myineqiii c(s,a) + \wt P_{s,a} V^{{*}}_{g} - 2 \sqrt{ \frac{\mathbb{V}(\wh P_{s,a}, V^{{*}}_{g}) \iota_{s,a} }{n(s,a)}}  - \frac{14 B \iota_{s,a}}{3 n(s,a)} \\
    &\myineqiv c(s,a) + \wh P_{s,a} V^{{*}}_{g} - 2 \sqrt{ \frac{\mathbb{V}(\wh P_{s,a}, V^{{*}}_{g}) \iota_{s,a} }{n(s,a)}} - \frac{14 B \iota_{s,a}}{3 n(s,a)} \\
    &\myineqv \underbrace{c(s,a) + P_{s,a} V^{{*}}_{g}}_{= Q^{{*}}_{g}(s,a)} - (B - B_{{*}})  \frac{14 \iota_{s,a}}{3 n(s,a)} \\
    &\le Q^{{*}}_{g} (s,a),
\end{align*}
where (i) is by definition of $\calG_1$ and choice of $c_3$, (ii) uses the first property of Lem.\,\ref{lemma_properties_f} and the induction hypothesis that $V^{(i)} \leq V^{{*}}_{g}$, (iii) uses the second property of Lem.\,\ref{lemma_properties_f} and assumption $B\ge \max\{ B_{{*}}, 1 \}$, (iv) uses Lem.\,\ref{lemma_properties_wtP}, (v) is by definition of $\calG_3$. Ultimately, for any $(s,a) \in \cK \times \cA$ and $s \neq g$,
\begin{align*}
    Q^{(i+1)}(s,a) \leq Q^{{*}}_{g}(s,a).
\end{align*}
Then for any $s \in \cK$ and $s \neq g$, we have $V^{(i+1)}(s) = \min\limits_{a\in\cA} Q^{(i+1)}(s,a) \leq \min\limits_{a\in\cA} Q^{{*}}_{g}(s,a) = V_{g}^{{*}}(s)$.

In addition, $V^{(i+1)}(g) = 0 = V^*_{g}(g)$, and we have
$$V^{(i+1)}(x) = \cRESET + V^{(i)}(s_0) \leq \cRESET + V^*_{g}(s_0) = V^*_{g}(x).$$
This completes the proof of this lemma.
\end{proof}

\section{Proof of Thm.\,\ref{theoremupSSP}}

Here we give a proof of Thm.\,\ref{theoremupSSP} and we focus on the fixed set $\cK^\dagger$ and our constructed MDP $M^\dagger = \langle \cK^\dagger, \cA, P^\dagger, c^\dagger, s_0 \rangle$. We denote $K = |\cK|$, and $K' = |\cK^\dagger| = K + 1$.

\begin{proofidea}
First we prove that Alg.\ref{algo3} solves AX problem (cf. Lem.~\ref{goodpolicy}), i.e., $\cK \supseteq \cSL$, and $\forall s \in \cK, V^{\pi_{s}}_s\left(s_{0}\right) \leq V^*_{\cK,s}(s_0) + \varepsilon L.$ By running DisCo algorithm with $\varepsilon = 1$, we have $\cSL \subseteq \cK \subseteq \cS_{2L}^\rightarrow$. To prove that each output policy $\pi_s$ is near-optimal, we observe that in the success round with respect to goal $s$, the average cost of executing $\pi_s$ (denoted as $\hat{\tau}$) in $\lambda$ episodes is no more than $V(s_0) + \epsilon L$. As we set $\lambda = \wt{O}(1/\epsilon^2)$, by concentration inequalities (cf. Lem.~\ref{probbound}), we can obtain that the expected cost of $\pi_s$ is close to the average cost $\hat{\tau}$, i.e., $V^{\pi_{s}}_s\left(s_{0}\right) \leq \hat{\tau} + \epsilon L \leq V(s_0) + 2\epsilon L$. By optimistic property of \VISGO (Lem.~\ref{lemma_opt}), we have $V(s_0) \leq V^*_{\cK,s}(s_0)$, i.e., our estimation of the value function $V(s_0)$ is no more than the optimal cost. Hence we obtain $V^{\pi_{s}}_s\left(s_{0}\right) \leq V^*_{\cK,s}(s_0) + 2\epsilon L \leq V^*_{\cK,s}(s_0) + \varepsilon L$ for all $s \in \cK$.


Then we bound the cumulative cost $C_T$. We first bound the total cost in DisCo algorithm and Alg.\,\ref{algo2}. By Lem.~\ref{theoremupDisCo}, with probability over $1 - \delta$, DisCo algorithm with $\varepsilon = 1$ uses no more than $\wt{O}(L^3S_{2L}^2A/\cmin^2)$ samples. In Alg.\,\ref{algo2}, for each state-action pair $(s,a) \in \cK \times \cA$, we collected $\wt{O}(L^2K/\cmin^2)$ samples. And to reach each $s \in \cK$, we executed the policy $\pi_s$, and the cost to reach $s$ is no larger than $\wt{O}(L)$ with high probability. Thus the total cost in Alg.\,\ref{algo2} can be bounded by $\wt{O}(L^3S_{2L}^2A/\cmin^2)$.

Now we will bound the cumulative cost of Alg.\,\ref{algo3} after running DisCo algorithm and Alg.\,\ref{algo2}. It's straightforward to show that the total cost in each round is bounded by $\wt{O}(L\lambda) = \wt{O}(L / \varepsilon^2)$. Hence to bound the cumulative cost $C_T$, we need only to bound the total number of rounds $r$. The number of success rounds is at most $K$. As the trigger condition holds for at most $\log_2(2T)$ times for each state-action pair $(s,a)$, the number of skipped rounds can be bounded by $K'A\log_2(2T)$ (where $T$ is the total number of samples we collected in Alg.~\ref{algo3}). Now we need only to bound the number of failure rounds $r_f$. 

To bound $r_f$, we borrow the idea from \citep{lim2012autonomous}. We first define the regret of an episode as the total cost in this episode minus our estimation of the optimal cost $V(s_0)$, and define the total regret as the sum of the regret in each episode (cf. Eq.~\ref{regret}). Then we will give both the upper bound and lower bound of the regret, where the upper bound scales as $\wt{O}(\sqrt{r_f})$ and the lower bound scales as $\wt{\Omega}(r_f)$. For the upper bound, we extend the techniques in \citep{tarbouriech2021stochastic} from classical SSP to multi-goal SSP, and we obtain an upper bound that scales as $\wt{O}({L}{\epsilon}^{-1}\sqrt{KAr_f} + {LKA}{\epsilon}^{-1} + LK^2A)$ (cf. Eq.~\ref{regretuprf}).

For the lower bound, as the regret in each failure round is at least $\lambda\epsilon L = \wt{\Omega}(L/\epsilon)$, we need only to give a lower bound for the regret of all the success rounds and skipped rounds (which can be negative). We observe that the regret in any round is larger than the total cost of executing policy $\tilde{\pi}$ in this round minus the expected cost of $\tilde{\pi}$, hence we can use concentration inequalities to bound the regret in all the success rounds and skipped rounds (which scales as $-\wt{O}({LKA}{\epsilon}^{-1})$). The lower bound of the total regret scales as $ \wt{\Omega} ({L}{\epsilon}^{-1}r_f - {LKA}{\epsilon}^{-1})$ (cf. Lem.~\ref{regretlowerbound}). Hence we can prove that the number of failure rounds $r_f = \wt{O}(KA + \varepsilon K^2A)$, and the total number of rounds $r = \wt{O}(S_{2L}A + \varepsilon S_{2L}^2A)$ (here we used $K \leq S_{2L}$).

As the cost in each round is bounded by $\wt{O}(L/\varepsilon^{2})$, the cumulative cost after running Alg.~\ref{algo2} is bounded by $\wt{O}(LS_{2L}A\varepsilon^{-2} + LS_{2L}^2A\varepsilon^{-1})$. And the cumulative cost in DisCo algorithm and Alg.~\ref{algo2} is bounded by $\wt{O}(L^3S_{2L}^2A\cmin^{-2})$, hence we complete the proof of Thm.\,\ref{theoremupSSP}.
\end{proofidea}

Now we give the full proof. We recall that we denote $V^*_g(s)$ as the expected cost of the optimal policy on MDP $\cM^\dagger$ to reach goal $g$ from state $s$, which equals to $V^*_{\cK,g}(s)$ for all $s \in \cK$.

First we prove the correctness of Alg.\,\ref{algo3}, i.e., with probability over $1-\delta$, Alg.\,\ref{algo3}  outputs a set $\cK \supseteq \cSL$ and a set of policies $\{\pi_s\}_{s \in \cK}$, such that
$$\forall s \in {\cK}, V_s^{\pi_{s}}\left(s_{0}\right) \leq V_{\cK,s}^*(s_0) + \varepsilon L.$$ The main intuition is that each policy $\pi_s$ has been tested for $\lambda$ times in a success round, and the average cost it takes to reach $s$ from $s_0$ is less than our estimate for optimal cost $V(s_0)$ plus $\epsilon L$. Thus by concentration inequalities, the expected cost of $\pi_s$ is close to optimal.

\newcommand{\roundcost}{\tau}

\newcommand{\goodpolicy}{Let $\{\pi_s\}_{s\in\cK}$ be the set of policies output by Alg.\ref{algo3}. With probability at least $1-{\delta}$, $V^{ \pi_{s}}_s(s_0) \leq V^*_s(s_0)+\varepsilon L$ for all $s\in\cK$.}
\begin{lemma}[\mainalgo~ Solves AX Problem]
\label{goodpolicy}
\goodpolicy
\end{lemma}
\begin{proof}
We fix any state $s \in \cK$. 
In any given round where the chosen target is $s$, let $\hat{\tau}_{k}$ be the total cost in the $k$-th episode of that round. Recall that for the algorithm to output a policy $\pi_{s}$, its empirical performance after $\lambda$ episodes must satisfy that 
$\hat{\tau} \leq V(s_0)+\epsilon L,$
where
$\hat{\tau}=\frac{\sum_{k=1}^{\lambda} \hat{\tau}_{k}}{\lambda}$ and $V$ is the output of \VISGO in that round. By optimism property (Lem.~\ref{lemma_opt}), when $\cK \subseteq \cS_{2L}^\rightarrow$, we have $V(s) \leq 2L+1$ for all $s \in \cK^\dagger$.

We define the random variable $\roundcost$ as the total cost it takes to reach the goal state $s$ from the start state $s_0$ when executing policy $\pi_{s}$, and we have $E(\roundcost) = V^{\pi_{s}}_s(s_0)$ by definition. We note that we have collected $\phi = \wt{\Omega}(L^2 K/\cmin^2)$ samples for each of the state-action pair $(s,a)$ (cf. Alg.~\ref{algo2}). By Lem.\,\ref{consterror}, under event $\calE$, the policy $\pi_{s}$ is proper, and we have $E(\roundcost) \leq 2V(s_0) \leq 4L+2$. Moreover, we have $d := \|V_s^{\pi_s}\|_\infty \leq 4L+2$. By Lem\,\ref{exptail}, we obtain $\Pr(\tau > m) \leq 2\exp(-m / 4d)$ for any $m > 0$.  
As we set
$$
\lambda = \lceil \frac{2048}{\epsilon^2}\ln^2(\frac{256}{\epsilon})\ln(\frac{2|\cK|}{\delta}) \rceil,
$$
by Lem.\,\ref{probbound}, we obtain that with probability at least $1 - \delta / (2K')$, we have
$$
V^{\pi_{s}}_s(s_0) = E(\tau) \leq \hat{\tau} + \epsilon L.
$$
We note that $ \hat{\tau} \leq V(s_0) + \epsilon L \leq V_s^*(s_0) + \epsilon L$ by the optimistic property (Lem.\,\ref{lemma_opt}). Hence, as we set $\epsilon = \varepsilon / 3$ in initial, we obtain that with probability at least $1 - \delta/(2K)$, $$
V^{\pi_{s}}_s(s_0) \leq V_s^*(s_0) + \varepsilon L.
$$
Finally, as there are at most $K$ states in total, and the event $\calE$ holds with probability at least $1 - \delta$, by the union bound and setting $\delta \rightarrow \delta/C$ in initial ($C$ is a large constant), the total success probability is at least $1-\delta$.
\end{proof}

Now we focus on bounding the cumulative cost of Alg.\ref{algo3}. 

First, we bound the total cost in DisCo algorithm and Alg.\,\ref{algo2}. Disco algorithm has sample complexity $\wt{O}(L^3S_{(1+\varepsilon)L}^2A/(\cmin^2\varepsilon^2))$, and when $\varepsilon = 1$, the total cost is bounded by $\wt{O}(L^3S_{2L}^2A/\cmin^2)$. In Alg.\,\ref{algo2}, we collect $\phi = \wt{O}(L^2K/\cmin^2)$ samples for each state-action pair $(s,a) \in \cK \times \cA$. To collect each sample  $(s,a,s',c)$, we executed a policy $\pi_s$ to reach the state $s$, and the expected cost $V_s^{\pi_s}(s_0) \leq 2L$. By Lem.\,\ref{exptail}, we obtain that with probability at least $1 - \delta$, for any state $s$, each time when the policy $\pi_s$ is executed, the total cost to reach $s$ from $s_0$ is no larger than $O(L \log (K / \delta))$. Therefore, the total cost of Alg.\,\ref{algo2} can be bounded by $O(KA\phi L \log (K / \delta)) = \wt{O}(L^3S_{2L}^2A/\cmin^2)$. We note that we used Lem.\,\ref{exptail} no more than $\phi KA$ times, the total failure probability is no more than $\phi KA\delta$. Substituting $\delta$ by $\delta / (2\phi KA)$ in the proof, we can obtain that the total cost of DisCo and Alg.\,\ref{algo2} is bounded by $\wt{O}(L^3S_{2L}^2A/\cmin^2)$ with probability $1-\delta$.

Then we bound the total cost of Alg.\ref{algo3} after running Alg.\,\ref{algo2}.

The key idea lies in bounding the "regret". We will use the regret to bound the total number of rounds. We first define the regret in the $k$-th episode of the $j$-th round. We denote $H^{j,k}$ as the number of steps it takes in the $k$-th episode of the $j$-th round. The regret in an episode $k$ is defined as
$$
(\sum\limits_{h=1}^{H^{j,k}} c_h^{j,k}) - V^{j}(s_0),
$$
where $c_h^{j,k}$ is the empirical cost in the $h$-th step in the $k$-th episode in the $j$-th round, and $V^{j}(s_0)$ is the value of $V(s_0)$ in the $j$-th round. Let $n_j$ be the total number of episodes executed in the $j$-th round, we define the regret in the $j$-th round as follows:
$$
\sum\limits_{k=1}^{n_j} ((\sum\limits_{h=1}^{H^{j,k}} c_h^{j,k}) - V^{j}(s_0)).
$$

Then we will define the total regret of Alg.\ref{algo3}. Let $r$ be the total number of rounds, $n_j$ be the total number of episodes executed in the $j$-th round, and $0 \leq n_j \leq \lambda$. Then we know that the total number of episodes in the whole process of Alg.\ref{algo3} is $M = \sum\limits_{j=1}^r n_j$. For notation convenience, we define $H^m$ as the number of steps it takes in the $m$-th episode of the whole process of Alg.\ref{algo3}, and denote $c_h^m$ as the empirical cost in the $h$-step of episode $m$. Finally we define the total regret of all the rounds as
\begin{eqnarray}\label{regret}
R := \sum\limits_{j=1}^r \sum\limits_{k=1}^{n_j} ((\sum\limits_{h=1}^{H^{j,k}} c_h^{j,k}) - V^{j}(s_0)) = \sum\limits_{m=1}^M ((\sum\limits_{h=1}^{H^m} c_h^m) - V^{m}(s_0)).
\end{eqnarray}

We will give both the upper bound and the lower bound of the regret. Here we give the upper bound.

\newcommand{\regretbound}{Under event $\cE$,  the total regret in $M$ episodes is at most
$$
R = \widetilde{O}(L\sqrt{KAM} + LK^2A).
$$}
\begin{lemma}[Upper Bound of Regret]
\label{regretbound}
\regretbound
\end{lemma}
This upper bound comes from the regret bound of the EB-SSP algorithm (cf. \citep{tarbouriech2021stochastic}), which solves the classical SSP problem with a single goal state $g$. 
To extend their result to multi-goal SSP, instead of only concentrating on $\wh{P}_{s,a} V_g^*$ for a single goal $g$ and one vector $V_g^*$, in our high probability event $\mathcal{E}_3$, we use concentration over $(\wh P_{s,a} - P^\dagger_{s,a}) V_g^{{*}}$ for all the goal states $g \in \cK$. Then following similar proof with Thm.3, \citep{tarbouriech2021stochastic}, we can obtain the regret upper bound in Lem.~\ref{regretbound}.

We note that the original form of the regret upper bound in Thm.3, \citep{tarbouriech2021stochastic} was $\wt{O}(B_* \sqrt{SAM} + BS^2A)$, where $B_* := \max\limits_{s \in \cS}V^*_g(s)$ in their work, $B$ is an upper bound of $B_*$ which is used in \VISGO, and $M$ is the number of episodes. In our Alg.~\ref{algo3}, we work on MDP $\cM^\dagger$, and all the states in $\cK$ are incrementally $2L$-controllable from $s_0$. Hence in our settings, $B_*:= \max\limits_{(s,g) \in \cK^\dagger\times\cK} V^*_g(s) \leq 2L + 1$, and the number of states in MDP $\cM^\dagger$ is $K' = K+1$. And in our Alg.~\ref{algo3}, we set $B = 10L$. Therefore, by setting $B_* = O(L), B = O(L), S = K+1$ in their regret bound $\wt{O}(B_* \sqrt{SAM} + BS^2A)$, we can obtain the regret bound $\widetilde{O}(L\sqrt{KAM} + LK^2A)$.


\newcommand{\numfail}{r_f}

We observe that there are at most $\wt{O}(KA)$ skipped rounds and $K$ success rounds. We denote by $\numfail$ the number of failure rounds, and we have the total number of episodes $M = \wt{O}((KA + \numfail)\lambda) = \wt{O}((KA + \numfail)/\epsilon^2).$ Thus the total regret in $r$ rounds can be bounded by $\numfail$ sublinearly:
\begin{eqnarray}\label{regretuprf}
R = \wt{O}(\frac{L}{\epsilon}\sqrt{KA\numfail} + \frac{LKA}{\epsilon} + LK^2A).
\end{eqnarray}

Then we gives the lower bound of the total regret in terms of the number of failure rounds $\numfail$.

\begin{lemma}[Lower Bound of Regret]\label{regretlowerbound}
With probability $1-\delta$, when $r = \wt{O}((KA)^2)$, the total regret in the first $r$ rounds is at least
$$
R = \wt{\Omega} (\frac{L\numfail}{\epsilon} - \frac{LKA}{\epsilon}),
$$
where $\numfail$ is the number of failure rounds in the $r$ rounds.
\end{lemma}
\begin{proof}
By the criterion of our performance check, in any failure round, we have $\hat{\tau}>V(s_0)+\epsilon L$, and in round $j$, we have $\hat{\tau} = \frac{1}{\lambda}\sum\limits_{k=1}^{n_j} (\sum\limits_{h=1}^{H^{j,k}} c_h^{j,k})$ by definition. Hence, in any failure round $j$, the regret is $ \lambda \hat{\tau} - n_j V^j(s_0) \geq \lambda (\hat{\tau} - V^j(s_0)) \geq \lambda \epsilon L = \wt{\Omega}(\frac{L\numfail}{\epsilon})$.

Then we focus on skipped rounds and success rounds. We denote $g^j$ as the goal state in the $j$-th round, and $\pi_j$ as the policy $\tpi$ in the $j$-th round, which is the greedy policy over the $Q$-function in the $j$-th round. We observe that the regret in any round $j$ satisfies
$$
\sum\limits_{k=1}^{n_j} ((\sum\limits_{h=1}^{H^{j,k}} c_h^{j,k}) - V^{j}(s_0)) \geq -L + \sum\limits_{k=1}^{n_j-1} ((\sum\limits_{h=1}^{H^{j,k}} c_h^{j,k}) - V_{g^j}^*(s_0)) \geq
-L + \sum\limits_{k=1}^{n_j-1} ((\sum\limits_{h=1}^{H^{j,k}} c_h^{j,k}) - V_{g^j}^{\pi_j}(s_0)),
$$
where we used the optimism property in Lem.\,\ref{lemma_opt}. We note that $ \sum\limits_{h=1}^{H^{j,k}} c_h^{j,k}$ is the empirical cost of policy $\pi_j$ in episode $k$, and we will use the concentration inequality to give a lower bound of the regret in round $j$. As the last episode in a skipped round can terminate before reaching the goal, we should take special considerations the last episode of each round. We directly use $-L$ to lower bound the regret of the last episode in round $j$. Then we denote $n = n_j - 1$, and focus on the previous $n$ episodes in round $j$.

Now we fix the round index $j$. We denote the random variable $\roundcost$ as the cost to reach $g^j$ from $s_0$, and we recall that $\hat{\tau}_k = \sum\limits_{h=1}^{H^{j,k}} c_h^{j,k}$. By Lem.\,\ref{probbound} with $d = 4L$, with probability at least $1 - \frac{\delta}{(KA)^2}$, we have 
$$
\sum\limits_{k=1}^n (\hat{\tau}_k - E(\tau)) \geq -2\Gamma\sqrt{n\ln(\frac{KA}{\delta})} \geq -2\Gamma\sqrt{\lambda\ln(\frac{KA}{\delta})} \geq -\wt{O}(\frac{L}{\epsilon}),
$$
where $\Gamma = \lfloor 8d \ln (64 / \epsilon) \rfloor$.
Thus the regret in any round $j$ is larger than $-\wt{O}(\frac{L}{\epsilon})$. As there are at most $\wt{O}(KA)$ skipped rounds and $K$ success rounds, we obtain that the total regret $R$ has the lower bound
$$
R = \wt{\Omega} (\frac{L\numfail}{\epsilon} - \frac{LKA}{\epsilon}).
$$
Now we bound the total failure probability. The number of rounds $r = \wt{O}((KA)^2)$, in each round the failure probability is at most $\frac{\delta}{(KA)^2}$, and the events $\calE$ fails with probability $\delta$. By replacing $\delta$ by $\delta / C$ throughout the proof ($C$ is a large constant), we obtain that the total failure probability is at most $\delta$.
\end{proof}

As the lower bound is linear in $\numfail$, and the upper bound is sublinear in $\numfail$, we can solve it and obtain that $\numfail = \wt{O}(KA + \epsilon K^2A)$, thus the total number of rounds can be bounded by $\wt{O}(KA + \epsilon K^2A)$.

To get the cumulative cost bound in Thm.\,\ref{theoremupSSP}, we need only to bound the cost in a round. In any round, we observe that except for the last episode, the average cost $\hat{\tau}$ for all the other episodes is no larger than $V(s_0) + \epsilon L \leq 2L$, thus the total cost in these episodes is no larger than $2L\lambda = \wt{O}(L/\epsilon^2)$. Also, we know that in the any episode, the expected cost of the policy $\tpi$ to reach the goal from $s_0$ is no larger than $2L$. Thus by Lem.\,\ref{exptail}, in any round, with probability at least $1 - \frac{\delta}{(KA)^2}$, the cost in the last episode is no larger than $\wt{O}(L).$ Hence, the total cost in each round is no larger than $\wt{O}(L/\epsilon^2)$. By multiplying it with $\wt{O}(KA + \epsilon K^2A)$ and using $K \leq S_{2L}$, the cumulative cost in Alg.\ref{algo3} can be bounded by $\wt{O}(LS_{2L}A/\varepsilon^2 + LS_{2L}^2A/\varepsilon + L^3S_{2L}A/\cmin^2)$, where the term $L^3S_{2L}A/\cmin^2$ comes from the subroutine of DisCo algorithm and Alg.\,\ref{algo2}. Hence we obtain the bound in Thm.\,\ref{theoremupSSP}.

Now we count the total failure probability. First, DisCo algorithm fails with probability $\delta$, the event $\mathcal{E}$ fails with probability ${\delta}$, and the lower bound of the total regret $R$ fails with probability $\delta$. And in the previous paragraph, to bound the cost in the last episode of each round using Lem.\,\ref{exptail}, the failure probability in each episode is at most $\frac{\delta}{(KA)^2}$. We observe that the total number of these failures
is no larger than the total number of rounds, and the total number of rounds can be bounded by $\wt{O}(KA + \epsilon K^2A)$, where we omit the logarithmic factors. Thus by setting $\delta \rightarrow \delta / C$ in the proof ($C$ is a large constant), we can bound the total failure probability by $\delta$, and the proof of Thm.\,\ref{theoremupSSP} is completed.

Here we briefly discuss the time complexity and space complexity of our algorithm $\mainalgo$. The time complexity scales as $\wt{O}(TK^3A^2(KA + \epsilon K^2A))$, where $T$ is the total number of samples collected, and $\wt{O}(KA + \epsilon K^2A)$ is the total number of rounds. The bottleneck is on the \VISGO procedure, and the time complexity of a \VISGO procedure is analyzed in Appendix G, \citep{tarbouriech2021stochastic}, which scales as $\wt{O}(TK^3A^2)$. The space complexity scales as $\tilde{O}(T+K^2A) = \tilde{O}(T)$, and the bottleneck is on storing the samples and the empirical model $\wh{P}$.

\section{Analysis of the Lower Bound}

\label{app_ax_lowerbound}

Here we discuss the lower bound of the autonomous exploration problem. We recall that our algorithm needs output a set $\cK \supseteq \cSL$ and a set of policies $\{\pi_s\}_{s\in\cK}$, and when $s\in\cSL$, the policy $\pi_s$  satisfies $V_s^{\pi_s}(s_0) \leq (1+\varepsilon)L$. Moreover, we note that in our proof of the lower bound, we allow the algorithm to output Markov policies, i.e. non-stationary and non-deterministic policies, which is defined in the next paragraph.

We recall the some basic concepts about the definition of a learning algorithm, and we use the notations in \citep{domingues2021episodic}. 
Let
$\mathcal{I}^{t}=(\mathcal{S} \times \mathcal{A})^{t-1} \times  \mathcal{S}$
be the set of all possible histories up to $t$ steps, i.e., be the set of tuples of the form
$
\left(s^{1}, a^{1}, s^{2}, a^{2}, \ldots,  s^{t}\right) \in \mathcal{I}^{t}.
$ Let $\Delta(\cA)$ be the set of probability distributions over the action space $\cA$, and $ \mathbb{N}^*$ be the set of positive integers. A Markov policy is a function $\pi : \cS \times \mathbb{N}^* \rightarrow \Delta(\cA)$ such that $\pi(a \mid s,h)$ denotes the probability of taking action $a$ in state $s$ at step $h$. And we note that the Markov policy $\pi$ is history-independent. 

A history-dependent policy is a family of functions denoted as $\algo \triangleq\left(\pi^{t}\right)_{t \geq 1}$, where $\pi^{t}: \mathcal{I}^{t} \rightarrow \Delta(\mathcal{A})$ describes the probability of taking action $a \in \cA$ after observing some history $i^t \in \mathcal{I}^{t}$.

Given an MDP $\cM = \langle \mathcal{S}, \mathcal{A},  p, c, s_0\rangle$,
a policy $\algo$ interacting with the MDP $\cM$ defines a stochastic process denoted by $\left(S^{t}, A^{t}\right)_{t \geq 1}$, where $(S^{t}, A^{t})$ is the state-action pair at time $t$. The Ionescu-Tulcea theorem ensures the existence of the  probability space $\left(\Omega, \mathcal{F}, \mathbb{P}_{\mathcal{M}}\right)$ such that 
$$
\mathbb{P}_{\mathcal{M}}\left[S^1=s\right]=\mathds{I}[s=s_0], \mathbb{P}_{\mathcal{M}}\left[S^{t+1}=s \mid A^{t}, I^{t}\right]=p\left(s \mid S^{t}, A^{t}\right), \text { and }  \mathbb{P}_{\mathcal{M}}\left[A^{t}=a \mid I^{t}\right]=\pi^{t}\left(a \mid I^{t}\right),
$$
where $\algo=\left(\pi^{t}\right)_{t \geq 1}$ and for any $t$,
$I^{t} \triangleq\left(S^{1}, A^{1}, S^2, A^2, \ldots S^{t}\right)$ is the random vector in $\mathcal{I}^{t}$ containing all state-action pairs observed up to step $t$. 
We denote the $\sigma$-algebra generated by $I^{t}$ as $\mathcal{F}^{t}$. And we denote by $\mathbb{P}_{\mathcal{M}}^{I^{T}}$ the measure of $I^{T}$ under $\mathbb{P}_{\mathcal{M}}$ as follows:
$$
\mathbb{P}_{\mathcal{M}}^{I^{T}}\left[i^{T}\right] \triangleq \mathbb{P}_{\mathcal{M}}\left[I^{T}=i^{T}\right]=\mathds{I}(s^1=s_0) \prod_{t=1}^{T-1} \pi^{t}\left(a^{t} \mid i^{t}\right) p\left(s^{t+1} \mid s^{t}, a^{t}\right).
$$

Then we denote $\mathbb{E}_{\mathcal{M}}$ as the expectation under $\mathbb{P}_{\mathcal{M}}$. Note that the dependence of $\mathbb{P}_{\mathcal{M}}$ and $\mathbb{E}_{\mathcal{M}}$ on the  policy $\boldsymbol{\pi}$ is denoted implicitly in the definition of $\mathbb{P}_{\mathcal{M}}$. We will denote them explicitly as $\mathbb{P}_{\boldsymbol{\pi}, \mathcal{M}}$ and $\mathbb{E}_{\boldsymbol{\pi}, \mathcal{M}}$ respectively when we need to stress $\algo$.

We recall that we define an algorithm for the AX problem as a tuple $(\algo, \tau, \cK, \{\pi_s\}_{s \in \cK})$, where $\algo$ is a history-dependent policy, $\tau$ is the stopping time chosen by the algorithm, $\cK$ and $\{\pi_s\}_{s \in \cK}$ are the output of the algorithm. And given the algorithm $\algo$ and the MDP $\mathcal{M}$ for AX problem, we can regard the number $\tau$, the set of states $\cK$, and the set of policies $\{\pi_s\}_{s \in \cK}$ as random variables on distribution $\mathbb{P}_{\boldsymbol{\pi}, \mathcal{M}}$.

Moreover, for any the MDP $\cM = \langle \mathcal{S}, \mathcal{A},  p, c, s_0\rangle$,  given a Markov policy ${\pi}$ with goal state $g$, we denote $V^{\pi}_{\cM, g}(s)$ as the expected cost of policy $\pi$ to reach state $g$ from state $s$ in MDP $\cM$. Formally, $V_{\cM,g}^{\pi}(s) = \mathbb{E}_{\pi, \cM}\left[\sum_{t=1}^{t_g^{\pi}\left(s \right)} c_{t}\left(s_{t}, \pi\left(s_{t}\right)\right)  \right],$ where $t_g^{\pi}\left(s \right):=\inf \left\{t \geq 0: s_{t+1}=g\right\}.$ 
And we denote $V^*_{\cM,g}(s)$ as the expected cost of the optimal policy $\pi$ to reach the goal state $g$ from the state $s$ on MDP $\cM$. 

Here we introduce the basic definitions and the technical lemmas used in our proof.

\begin{definition}[KL divergence]
The Kullback-Leibler divergence between two distributions $\mathbb{P}_{1}$ and $\mathbb{P}_{2}$ on a measurable space $(\Omega, \mathcal{G})$ is defined as
$$
\mathrm{KL}\left(\mathbb{P}_{1}, \mathbb{P}_{2}\right) \triangleq \int_{\Omega} \log \left(\frac{\mathrm{d} \mathbb{P}_{1}}{\mathrm{~d} \mathbb{P}_{2}}(\omega)\right) \mathrm{d} \mathbb{P}_{1}(\omega),
$$
if $\mathbb{P}_{1} \ll \mathbb{P}_{2}$ and $+\infty$ otherwise. For Bernoulli distributions, we define $\forall(p, q) \in[0,1]^{2}$,
$$\mathrm{kl}(p, q) \triangleq \mathrm{KL}(\mathcal{B}(p), \mathcal{B}(q))=p \log \left(\frac{p}{q}\right)+(1-p) \log \left(\frac{1-p}{1-q}\right).$$
\end{definition}

\newcommand{\keylemma}{
Let $\mathcal{M}$ and $\mathcal{M}^{\prime}$ be two MDPs that are identical except for their transition probabilities, denoted by $p$ and $p^{\prime}$, respectively. Assume that we have $\forall(s, a)$, $p(\cdot \mid s, a) \ll p^{\prime}(\cdot \mid s, a) .$ Then, for any stopping time $\tau$ with respect to $\left(\mathcal{F}^{t}\right)_{t \geq 1}$ that satisfies $\mathbb{P}_{\mathcal{M}}[\tau<\infty]=1$,
$$
\mathrm{KL}\left(\mathbb{P}_{\mathcal{M}}^{I^{\tau}}, \mathbb{P}_{\mathcal{M}^{\prime}}^{I^{\tau}}\right)=\sum_{s \in \mathcal{S}} \sum_{a \in \mathcal{A}}  \mathbb{E}_{\mathcal{M}}\left[N_{ s, a}^{\tau}\right] \mathrm{KL}\left(p(\cdot \mid s, a), p^{\prime}(\cdot \mid s, a)\right),
$$
where $N_{s, a}^{\tau} \triangleq \sum_{t=1}^{\tau} \mathbbm{1}\left\{\left(S^{t}, A^{t}\right)=(s, a)\right\}$ and $I^{\tau}$ is the
random vector representing the history of $\tau$ samples.
}
\begin{lemma}[Lemma 5, \cite{domingues2021episodic}, modified]
\label{keylemma}
\keylemma
\end{lemma}

\newcommand{\randomz}{
Consider a measurable space $(\Omega, \mathcal{F})$ equipped with two distributions $\mathbb{P}_{1}$ and $\mathbb{P}_{2}$. For any $\mathcal{F}$-measurable function $Z: \Omega \rightarrow[0,1]$, we have
$$
\mathrm{KL}\left(\mathbb{P}_{1}, \mathbb{P}_{2}\right) \geq \mathrm{kl}\left(\mathbb{E}_{1}[Z], \mathbb{E}_{2}[Z]\right),
$$
where $\mathbb{E}_{1}$ and $\mathbb{E}_{2}$ are the expectations under $\mathbb{P}_{1}$ and $\mathbb{P}_{2}$ respectively.
}
\begin{lemma}[Lemma 1, \cite{garivier2019explore}]
\label{randomz}
\randomz
\end{lemma}

\newcommand{\klineq}{
For any $p,q \in (0,\frac{1}{2}]$, $ \mathrm{kl}(p,q) \leq {2(p-q)^2}/{q}$.
}
\begin{lemma}
\label{klineq}
\klineq
\end{lemma}

\newcommand{\klineqtwo}{
For any $p,q \in [0,1]$,  $\mathrm{kl}(p, q) \geq -(1-p) \log \left({1-q}\right)-\log (2)$.
}
\begin{lemma}[Lemma 15, \citep{domingues2021episodic}]
\label{klineqtwo}
\klineqtwo
\end{lemma}

\newcommand{\cMsa}{\cM_{(s,a)}}
\newcommand{\psa}{p_{s,a}}
\newcommand{\gsa}{g_s^a}
\newcommand{\cSp}{\cS'}
\newcommand{\cSd}{\cS^\dagger}
\newcommand{\cAp}{\cA'}
\newcommand{\cD}{\mathcal{D}}
\newcommand{\cMd}{\mathcal{M}_d}

Now we construct a family of adversarial MDPs to obtain the lower bound of sample complexity. 

\para{The construction of hard MDPs with general $S_L$} Now we fix $L, S, A, S_L, \varepsilon, \cmin$ such that $L > 4$, $S > 8$, $A > 4$, $4 \leq S_L \leq \min\{(A-1)^{\lfloor\frac{L}{2}\rfloor}, S \}$, $0 < \varepsilon < \frac{1}{4}$, and $0 < \cmin \leq 1$. 

We first construct an MDP $\cM_{0}' = \langle \cS, \cA, p_0', c', s_0 \rangle$ with $ |\cS| =  S_L - 1$ states and $ |\cA| =  A$ actions ($\cM_{0}'$ does not contain the goal state $g$ in Fig.~\ref{FigureLB2}). As is illustrated in  Fig.~\ref{FigureLB2}, the construction on $\cM_0'$ follows a tree structure. This is inspired from \citep{domingues2021episodic}. We denote $\cS'$ as all the leaf states, and for any $s \notin \cS'$ and $a \in \cA$ ($a \neq \RESET$), we set $c(s,a) = 1$ with probability $1$, and the transition $p_0'(\cdot | s,a)$ is deterministic, i.e., taking any action $a$ at a non-leaf node $s$ will transit to one of its son $s' \in \cS$ with probability $1$. 

As $S_L \leq (A-1)^{\lfloor\frac{L}{2}\rfloor}$, there exists a tree structure with depth $d_0 \leq L/2$ ($d_0 \in \mathbb{N}$), such that the number of leaves $|\cS'| \geq \frac{S_L}{2}$, and $V^*_{\cM_0', s}(s_0) = d_0$ for all $s \in \cS'$, i.e., all the leaf nodes can be reached within $d_0$ steps from $s_0$. 
Hence in MDP $\cM_0'$, all the states in $\cS$ are incrementally $d_0$-controllable. And we denote $d_1 = L - d_0$.

\begin{figure}[ht]
\centering

\begin{tikzpicture}[->,>=stealth',shorten >=1pt,auto,node distance=2.8cm,
semithick]
\tikzstyle{every state}=[fill=white,text=black]

\node[state]                (A)   {$s_0$};


\node[state,fill=black]                (N1)[below left=1cm and 3cm of A]    {}; 
\node[state,fill=black]                (N2)[below right=1cm  and 3cm of A]    {}; 

\path 
	(A) edge  (N1)
	(A) edge  (N2);

\node[state,fill=white]     (N3)[below  left of=N1]    {$s_1$}; 
\node[state,fill=white]     (N4)[below  right of=N1]    {$s_2$}; 

\path 
	(N1) edge  (N3)
	(N1) edge  (N4);

\node[state,fill=white]     (N5)[below  left  of=N2]    {$s_3$}; 
\node[state,fill=white]     (N6)[below  right of=N2]    {$s_4$}; 

\path 
	(N2) edge (N5)
	(N2) edge (N6);

	
\node[state, fill=white]    (GOOD) [below = 5cm of A] {$g$};


\path
	(N3)  edge[dashed]   node[above left=0.5cm and 0.25cm]{$\frac{ \cmin}{(1+6\varepsilon)d_1}$} (GOOD)
	
	(N3) edge [loop below, dashed] node{$1-\frac{\cmin}{(1+6\varepsilon)d_1}$} (N3)
	
	
	(N4)  edge[dashed]   node[below]{} (GOOD)
	(N4)  edge [loop below, dashed] node{} (N4)
	
	(N5)  edge[dashed]   node[below]{} (GOOD)
	(N5)  edge [loop below, dashed] node{} (N5)
	
	(N6)  edge[dashed]   node[below]{} (GOOD)
	(N6)  edge [loop below, dashed] node{} (N6)
	;

\path 
	(N4)  edge[blue, bend left]   node[above =0.3cm ]{$\textcolor{blue}{\frac{\cmin}{d_1}} $} (GOOD)
	(N4) edge [blue, loop above] node{$1-\frac{\cmin}{d_1}$} (N4)
	
	(GOOD)  edge [loop below] (GOOD)  
	;
\end{tikzpicture}
\caption{Illustration of the hard MDP with general $S_L$.}
\label{FigureLB2}
\end{figure}
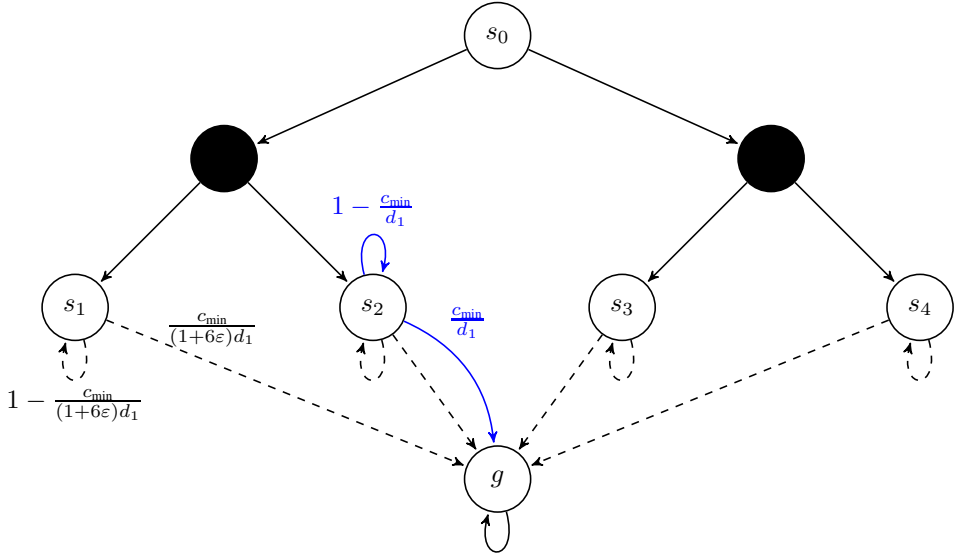

Then we construct the MDP $\cM_0 = \langle \cS \cup \{g\},\cA,p_0,c,s_0 \rangle$ based on $\cM_0'$ by adding a new state $g$. The state $g$ can only be reached from all the leaf nodes $s \in \cS'$. And for any leaf state $s \in \cS'$ and any action $a \in \cA$ ($a \neq \RESET$), we set $c(s,a) = \cmin$, and
\begin{align*}
p_0(g | s,a) = \frac{ \cmin}{(1+6\varepsilon)d_1}, \quad
p_0(s | s,a) = 1 - \frac{ \cmin}{(1+6\varepsilon)d_1}.
\end{align*}

Finally, we set $p_0 (g | g,a) = 1$ for any action $a \neq \RESET$. 

In this way, we have $V_{\cM_0,s}^*(s_0) = d_0$ for any $s \in \cS'$, and $V_{\cM_0,g}^*(s_0) = d_0 + (1+6\varepsilon)d_1 > (1 + \varepsilon) L$. Hence $g \notin \cSL$ for MDP $\cM_0$. Also, we note that in MDP $\cM_0$, with probability $1$, the goal state $g$ cannot be reached from $s_0$ within $d_0$ steps.

Now we will construct other adversarial MDPs based on $\cM_0$. We choose any $(s^*,a^*) \in \cS' \times \cA$ ($a^* \neq \RESET$), and we define the MDP $\cM_{(s^*,a^*)} = \langle \cS \cup \{g\},\cA,p_{(s^*,a^*)},c,s_0 \rangle$ by slightly increasing $p_0(g| s^*, a^*)$, i.e., we set
\begin{align*}
p_{(s^*,a^*)}(g | s^*,a^*) = \frac{\cmin}{d_1}, \quad
p_{(s^*,a^*)}(s^* | s^*,a^*) = 1 - \frac{\cmin}{d_1}.
\end{align*}
In this way, we have $V_{\cM_0,g}^*(s_0) = d_0 + d_1 = L$. Hence $g \in \cSL$ for MDP $\cM_{(s^*,a^*)}$.

Finally, we define the family of our adversarial MDPs as $\{\cM_0\} \cup \{\cMsa\}_{(s,a)\in\cSp\times\cA}$. We note that for each MDP $\cMsa$, its $|\cSL| = S_L$, and it satisfies Asmp.\,\ref{assa}. Also, for the MDP $\cM_0$, its $|\cSL| = S_L - 1$, and it also satisfies Asmp.\,\ref{assa}. Thus the family is valid for the AX problem.

We note that in MDP $\cM_{(s^*,a^*)}$, for any Markov policy $\pi$,
\begin{align*}
    V_{\cM_{(s^*,a^*)},g}^{\pi}(s_0) &= \mathbb{E}_{\pi, \cM_{(s,a)}}\left[\sum_{t=1}^{t_g^{\pi}\left(s \right)} c_{t}\left(s_{t}, \pi\left(s_{t}\right)\right)  \right] \\
    &= \mathbb{E}_{\pi, \cM_{(s^*,a^*)}}\left[\sum_{t=1}^{t_g^{\pi}\left(s \right)} c_{t}\left(s_{t}, \pi\left(s_{t}\right)\right) \mid (s_{d_0}, a_{d_0}) = (s^*,a^*) \right] \mathbb{P}_{\pi, \cM_{(s^*,a^*)}}[(s_{d_0}, a_{d_0}) = (s^*,a^*)] \\
    &+ \mathbb{E}_{\pi, \cM_{(s^*,a^*)}}\left[\sum_{t=1}^{t_g^{\pi}\left(s \right)} c_{t}\left(s_{t}, \pi\left(s_{t}\right)\right) \mid (s_{d_0}, a_{d_0}) \neq (s^*,a^*) \right] \mathbb{P}_{\pi, \cM_{(s^*,a^*)}}[(s_{d_0}, a_{d_0}) \neq (s^*,a^*)] 
\end{align*}

With probability $1$, we need at least $d_0$ steps to reach any of the leaf state.  And the expected cost to reach $g$ from state-action pair $(s^*, a^*)$ is $d_1$.

Hence we have $\mathbb{E}_{\pi, \cM_{(s^*,a^*)}}\left[\sum_{t=1}^{t_g^{\pi}\left(s \right)} c_{t}\left(s_{t}, \pi\left(s_{t}\right)\right) \mid (s_{d_0}, a_{d_0}) = (s^*,a^*) \right] \geq d_0 + d_1 = L$.

Also, when $(s_{d_0}, a_{d_0}) \neq (s^*, a^*)$, the expected cost to reach $g$ from state-action pair $(s_{d_0}, a_{d_0})$ is at least $(1+6\varepsilon)d_1$.

Hence we have $\mathbb{E}_{\pi, \cM_{(s^*,a^*)}}\left[\sum_{t=1}^{t_g^{\pi}\left(s \right)} c_{t}\left(s_{t}, \pi\left(s_{t}\right)\right) \mid (s_{d_0}, a_{d_0}) \neq (s^*,a^*) \right] \geq d_0 + (1+6\varepsilon)d_1 \geq (1 + 3\varepsilon) L$.

Therefore, if $V_{\cM_{(s^*,a^*)},g}^{\pi}(s_0) \leq (1+\varepsilon) L$, we have $\mathbb{P}_{\pi, \cM_{(s^*,a^*)}}[(s_{d_0}, a_{d_0}) = (s^*,a^*)] \geq 2/3$. We will use it in our proof of Thm.\,\ref{lowerbound}.

Now we give our proof of Thm.\,\ref{lowerbound} through the adversarial family of MDPs.
Here we use the techniques of Thm. 7 in  \citep{tarbouriech2020improved}. 

\begin{proof}
We denote by $\mathbb{P}_{\left( s^{*}, a^{*}\right)} \triangleq \mathbb{P}_{\algo, \mathcal{M}_{\left( s^{*}, a^{*}\right)}}$ and
$\mathbb{E}_{\left( s^{*}, a^{*}\right)} \triangleq \mathbb{E}_{\algo, \mathcal{M}_{\left( s^{*}, a^{*}\right)}}$ the probability measure and expectation in the MDP $\mathcal{M}_{\left( s^{*}, a^{*}\right)}$ by
following $\algo$ and by $\mathbb{P}_{0}$ and $\mathbb{E}_{0}$ the corresponding operators in the $\operatorname{MDP} \mathcal{M}_{0}$. We fix any algorithm $(\algo, \tau, \cK, \{\pi_s\}_{s \in \cK})$ that solves the AX problem. We will prove that when working on the MDP $\cM_0$, the algorithm will cost at least $\Omega(\frac{L S_L A}{  \cmin\varepsilon^2} \log \frac{1}{\delta})$ samples in expectation, i.e. $$\mathbb{E}_{0}[\tau] = \Omega(\frac{L S_L A}{  \cmin\varepsilon^2} \log \frac{1}{\delta}),$$ which yields that the lower bound of the total cost is $\Omega(\frac{L S_L A}{\varepsilon^2} \log \frac{1}{\delta})$.


Now we fix the state-action pair $(s^{*},a^*) \in \cSp \times \cA$ ($a^* \neq \RESET$). Also, we denote the random variable $N_{\left(s, a\right)}^{\tau}$ as the number of samples that the algorithm takes at the state-action pair $(s,a) \in \cS \times \cA$.  For any $\mathcal{F}^{\tau}$-measurable random variable $Z$ taking values in $[0,1]$, we have
\begin{align*}
&\quad \mathbb{E}_{0}\left[N_{\left(s^{*}, a^{*}\right)}^{\tau}\right] \frac{ 144\cmin\varepsilon^{2}}{L} \\
&\mygineeqa \  \mathbb{E}_{0}\left[N_{\left(s^{*}, a^{*}\right)}^{\tau}\right] \mathrm{kl}\left( \frac{\cmin}{(1 + 6\varepsilon) d_1}, \frac{\cmin}{d_1}\right) \\
&\myeqb \ \mathrm{KL}\left(\mathbb{P}_{0}^{I^{\tau}}, \mathbb{P}_{\left(s^{*}, a^{*}\right)}^{I^{\tau}}\right) \\
&\mygineeqc{} \    \operatorname{kl}\left(\mathbb{E}_{0}[Z], \mathbb{E}_{\left( s^{*}, a^{*}\right)}[Z]\right),
\end{align*}
 where (a) uses Lemma \ref{klineq} and $ d_1 \geq {L}/{2}$; (b) uses Lemma \ref{keylemma}; (c) uses Lemma \ref{randomz}.


For any $(s,a) \in \cSp \times \cA$, we define the event $Z_{s,a} = \mathbbm{1}\{\text{The algorithm's output satisfies } g \in \cK \text{ and } V^{\pi_g}_{\cM_{(s,a)},g}(s_0) \leq (1 + \varepsilon) L\}$. And we set the event $Z = Z_{s^*,a^*}$. We note that $Z_{s,a}$ can be viewed as a random event on distribution  $\mathbb{P}_{(s,a)}$, and can also be viewed as a random event on distribution $ \mathbb{P}_{0}$ (i.e., $\mathbb{P}_{\algo, \cM_0}$).

First we focus on distribution $\mathbb{P}_{\left( s^{*}, a^{*}\right)}$. We observe that as the algorithm $(\algo, \tau, \cK, \{\pi_s\}_{s \in \cK})$ solves the AX problem, when working on the MDP $\cM_{(s^*,a^*)}$, with probability at least $1-\delta$, its output should satisfy $g \in \cK$ and the expected cost of the policy $\pi_{g}$ to reach $g$ from state $s_0$ is no more than $(1 + \varepsilon) L$. Therefore, for any $(s^{*}, a^{*}) \in \cSp \times \cA$ ($a^* \neq \RESET$), we have $$\mathbb{P}_{\left( s^{*}, a^{*}\right)} [Z_{s^*,a^*}] \geq 1-\delta.$$

Then we focus on probability distribution $\mathbb{P}_0$ (i.e., $\mathbb{P}_{\algo, \cM_0}$). We recall that the event $Z_{s,a}$ implies $\mathbb{P}_{\pi_g, \cM_0}[(s_{d_0}, a_{d_0}) = (s,a)] \geq 2/3$. And for any two distinct state-action pairs $(s,a)$ and $(s',a')$, the event $\mathbb{P}_{\pi_g, \cM_0}[(s_{d_0}, a_{d_0}) = (s,a)] \geq 2/3$ and the event $\mathbb{P}_{\pi_g, \cM_0}[(s_{d_0}, a_{d_0}) = (s',a')] \geq 2/3$ are mutually exclusive. Hence $Z_{s,a}$ and $Z_{s',a'}$ are mutually exclusive on $\mathbb{P}_0$, and we have 
$$\sum_{(s,a) \in \cS' \times \cA}\mathbb{P}_{0}[Z_{s,a}] \leq 1.$$

We recall that we set $Z = Z_{s^*,a^*}$, and we can obtain
\begin{align*}
\mathrm{kl}\left(\mathbb{E}_{0}[Z], \mathbb{E}_{\left(s^{*}, a^{*}\right)}[Z]\right)&=\mathrm{kl}\left(\mathbb{P}_{0}\left[  Z_{s^*,a^*} \right], \mathbb{P}_{\left( s^{*}, a^{*}\right)}\left[  Z_{s^*,a^*} \right]\right)\\
&\mygineeqa \left(1-\mathbb{P}_{0}\left[ Z_{s^*,a^*} \right]\right) \log \left(\frac{1}{1-\mathbb{P}_{\left( s^{*}, a^{*}\right)}\left[ Z_{s^*,a^*} \right]}\right)-\log (2)\\
&\mygineeqb\left(1-\mathbb{P}_{0}\left[ Z_{s^*,a^*} \right]\right) \log \left(\frac{1}{\delta}\right)-\log (2),
\end{align*}
where (a) uses Lem.~\ref{klineqtwo}; (b) uses that $\mathbb{P}_{\left( s^{*}, a^{*}\right)}\left[Z_{s^*,a^*} \right] \geq 1 - \delta$.
Therefore, we have
$$
\mathbb{E}_{0}\left[N_{\left(s^{*}, a^{*}\right)}^{\tau}\right] \geq \frac{L}{144\cmin\varepsilon^2}(\left(1-\mathbb{P}_{0}\left[ Z_{s^*,a^*} \right]\right) \log \left(\frac{1}{\delta}\right)-\log (2)).
$$
We recall that $\sum_{(s,a) \in \cS' \times \cA}\mathbb{P}_{0}[Z_{s,a}] \leq 1.$ Thus summing up all the state-action pairs $\left(s^{*}, a^{*}\right)\in \cSp \times \cA$, we can obtain that
$$
\sum\limits_{(s^*,a^*) \in \cS' \times \cA} \mathbb{E}_{0}\left[N_{\left(s^{*}, a^{*}\right)}^{\tau}\right] \geq  \frac{L}{144\cmin\varepsilon^2}(\left( |\cS'| |\cA| -1\right) \log \left(\frac{1}{\delta}\right)-\log (2)|\cS'||\cA|).
$$

Hence provided that $|\cS'| \geq \frac{S_L}{2}$, $L > 4$, $S > 8$, $A > 4$, $4 \leq S_L \leq \min\{(A-1)^{\lfloor\frac{L}{2}\rfloor},S\}$, $0 < \varepsilon < \frac{1}{4}$, and $0 < \delta < \frac{1}{16}$, we can eventually obtain the lower bound of the total number of steps $\tau$,
$$
\mathbb{E}_{0}[\tau] = \sum\limits_{(s,a) \in \cS \times \cA}\mathbb{E}_{0}[N_{\left(s, a\right)}^{\tau}] \geq  \sum_{\left(s^{*}, a^{*}\right) \in \cSp \times \cA}  \mathbb{E}_{0}\left[N_{\left(s^{*}, a^{*}\right)}^{\tau}\right] \geq \Omega(\frac{L S_L A}{\cmin \varepsilon^2} \log \frac{1}{\delta}).
$$
\end{proof}

\section{Lower Bounds for Multi-goal SSP}\label{sect_lowerbound_MGSSP}
Here we formulize the lower bound for the multi-goal SSP problem. First we define an algorithm for the multi-goal SSP problem with goal space $\mathcal{G}$ as a triple $(\algo, \tau, \{\pi_s\}_{s \in \cG})$, which means the algorithm executes a history-dependent policy $\algo$, and returns a set of policies $\{\pi_s\}_{s \in \cG}$ after sampling $\tau$ times. Also, we allow $\pi_s$ to be Markov policies. And we release the multi-goal SSP problem in this way: we only require the algorithm output policies $\pi_s$ such that $V^{\pi_{s}}_s\left(s_{0}\right) \leq (1+\varepsilon)L$.

\begin{definition}
    {An algorithm $(\algo, \tau, \{\pi_s\}_{s \in \cS})$ is $(\varepsilon, \delta, L)$-PAC for multi-goal SSP problem on MDP $M$ with goal space $\mathcal{G} \subseteq \cS$, if with probability over $1 - \delta$, the algorithm returns a set of policies $\{\pi_s\}_{s \in \mathcal{G}}$ after $\tau$ steps, such that $\forall s \in \mathcal{G}, V^{\pi_{s}}_s\left(s_{0}\right) \leq (1+\varepsilon)L$.}
\end{definition}

{Then for any real numbers $L, \cmin$ and positive integers $S,A$, we define a class of MDPs $\mathfrak{M}_{\text{MSSP}}(L,S)$ as follows: $\mathfrak{M}_{\text{MSSP}}(L,S)$ contains all the MDPs $M = \langle \mathcal{S}, \mathcal{A}, P, c, s_0 \rangle$, such that $|\cS| \leq S$, $|\cA| \leq A$, $c(s,a) \in [\cmin,1]$ for all $(s,a) \in \cS \times \cA$, and $M$ satisfies Asmp.~\,\ref{assa} and $\cSL = \cS$.}

We remark that our constructed adversarial examples (cf. Fig.~\ref{FigureLB2}) for the autonomous exploration problem can also be applied to multi-goal SSP using the similar proof with Thm.\,\ref{lowerbound}. Thus we obtain the following lower bound for multi-goal SSP, which implies that our Alg.\,\ref{algo3} is also minimax for multi-goal SSP problem. See Appendix~\ref{app_ax_lowerbound} for more details.

\begin{theorem}
\label{lowerboundSSP}
Assume that $L > 4$, $A > 4$, $8 < S \leq (A - 1)^{\lfloor\frac{L}{2}\rfloor}$, $0 < \varepsilon < \frac{1}{4}$, $0 < \delta < \frac{1}{16}$, and $0 < \cmin \leq 1$. 
Then for any algorithm $(\algo, \tau, \cK, \{\pi_s\}_{s \in \cK})$ that is $(\varepsilon, \delta, L)$-PAC for multi-goal SSP problem on any MDP $M \in \mathfrak{M}_{\text{MSSP}}(L,S)$ with any goal space $\mathcal{G} \subseteq \cS$, there exists an MDP $\cM \in \mathfrak{M}_{\text{MSSP}}(L,S)$ such that 
$$
\mathbb{E}_{\algo, \cM}[\tau] =  \Omega(\frac{L S A}{\cmin \varepsilon^2} \log \frac{1}{\delta}).
$$
\end{theorem}

We note that in our construction of adversarial examples (cf. Fig.~\ref{FigureLB2}) and in our proof of Thm.\,\ref{lowerbound}, we only involved one goal state $g$. Hence we can also prove that for the classical single-goal SSP problem with $\calG = \{g\}$, learning a policy $\pi_g$ such that $V_g^{\pi_g}(s_0) \leq (1+\varepsilon)L$ also requires $\Omega(\frac{L S A}{\cmin \varepsilon^2} \log \frac{1}{\delta})$ samples, and the lower bound of cumulative cost scales as $\Omega({L S A}{ \varepsilon^{-2}} \log \frac{1}{\delta})$.

\end{document}